\documentclass[journal]{IEEEtran}
\usepackage{amsmath,amssymb,amsthm}
\usepackage{algorithm}
\usepackage{algorithmic}
\usepackage{array}
\usepackage{bm}
\usepackage{color}
\usepackage{diagbox}
\usepackage{epsfig}
\usepackage{multirow}
\usepackage{subfigure}
\usepackage{stackrel}
\usepackage{tabularx}
\usepackage[noblocks]{authblk}

\newtheorem{theorem}{Theorem}
\newtheorem{lemma}{Lemma}

\newcommand{\PreserveBackslash}[1]{\let\temp=\\#1\let\\=\temp}
\newcolumntype{C}[1]{>{\PreserveBackslash\centering}p{#1}}
\newcolumntype{R}[1]{>{\PreserveBackslash\raggedleft}p{#1}}
\newcolumntype{L}[1]{>{\PreserveBackslash\raggedright}p{#1}}

\begin{document}

\title{OPML: A One-Pass Closed-Form Solution for Online Metric Learning}

%      Yang~Gao,%~\IEEEmembership{Member,~IEEE,}
 %       Lei~Wang,%~\IEEEmembership{Member,~IEEE,}
  %      Luping~Zhou,%~\IEEEmembership{Member,~IEEE,}
   %     Jing~Huo,%~\IEEEmembership{Member,~IEEE,}
    %    and~Yinghuan~Shi,%~\IEEEmembership{Member,~IEEE}% <-this % stops a space
%\thanks{}% <-this % stops a space
%\thanks{}% <-this % stops a space
%\thanks{}}

\author[1]{Wenbin Li
\thanks{Email addresses: liwenbin.nju@gmail.com (Wenbin Li), gaoy@nju.edu.cn (Yang Gao), leiw@uow.edu.au (Lei Wang), lupingz@uow.edu.au (Luping Zhou), huojing1989@gmail.com (Jing Huo), syh@nju.edu.cn (Yinghuan Shi).}}
\author[1]{Yang Gao}
\author[2]{Lei Wang}
\author[2]{Luping Zhou}
\author[1]{Jing Huo}
\author[1]{Yinghuan Shi}
\affil[1]{National Key Laboratory for Novel Software Technology, Nanjing University, China}
\affil[2]{School of Computing and Information Technology, University of Wollongong, Australia}




\maketitle

\begin{abstract}
To achieve a low computational cost when performing online metric learning for large-scale data, we present a one-pass closed-form solution namely OPML in this paper. Typically, the proposed OPML first adopts a one-pass triplet construction strategy, which aims to use only a very small number of triplets to approximate the representation ability of whole original triplets obtained by batch-manner methods. Then, OPML employs a closed-form solution to update the metric for new coming samples, which leads to a low space (i.e., $O(d)$) and time (i.e., $O(d^2)$) complexity, where $d$ is the feature dimensionality. In addition, an extension of OPML (namely COPML) is further proposed to enhance the robustness when in real case the first several samples come from the same class (i.e., cold start problem). In the experiments, we have systematically evaluated our methods (OPML and COPML) on three typical tasks, including UCI data classification, face verification, and abnormal event detection in videos, which aims to fully evaluate the proposed methods on different sample number, different feature dimensionalities and different feature extraction ways (i.e., hand-crafted and deeply-learned). The results show that OPML and COPML can obtain the promising performance with a very low computational cost. Also, the effectiveness of COPML under the cold start setting is experimentally verified.
\end{abstract}


% \fi
\IEEEpeerreviewmaketitle

\section{Introduction}
\IEEEPARstart{I}{n} computer vision and machine learning, learning a meaningful distance/similarity metric on the original feature presentation of samples, with the given distance constraints (either pairwise similar/dissimilar distance constraints or triplet based relative distance constraints) at the same time, is usually regarded as a crucial and challenging problem, which has been actively studied over the decades. According to the different measure functions (e.g., Mahalanobis distance function and bilinear similarity function), the current metric learning methods can be roughly classified into two categories, i.e., \emph{Mahalanobis distance-based methods} and \emph{bilinear similarity-based methods}. The first class, Mahalanobis distance-based methods, refers to learning a pairwise real-valued distance function, which is parameterized by a symmetric Positive Semi-Definite (PSD) matrix. The second class, bilinear similarity-based methods, aims to learn a form of bilinear similarity function which does not need to impose the PSD constraint on learned metrics.

Recently, instead of batch manner, learning the metric in an online manner, which refers to online metric learning (OML), has attracted lots of interests, with the goal of learning a discriminative metric with partially known sampled data for efficiently dealing with large-scale learning problem. Generally, to satisfy the online processing speed for large-scale learning problem, OML methods are required to well tackle the following two core issues: (1) how to fast construct the triplet (or pair) in the original data, especially for the large-scale data, and (2) how to fast update the metric with the new coming samples in a real time manner.

For fast triplet (or pair) construction (first issue), existing OML methods usually assume that pairwise or triplet constraints can be obtained in advance \cite{shalev2004online}, or by employing the random sampling strategy to reduce the size of triplets \cite{chechik2010large}. However, in real applications, it is usually infeasible to access the entire training set at a time, especially when the training set is relative large, constructing the constraints will be both time- and space-consuming. To this end, we propose a novel one-pass triplet construction strategy to rapidly construct triplets in an online manner. In particular, the strategy selects two latest samples from both the same and different classes of currently available samples respectively, to construct a triplet. Compared with Online Algorithm for Scalable Image Similarity (OASIS) \cite{chechik2010large}, which utilizes a random sampling strategy and stores the entire training data in memory with space complexity of $O(md)$ ($d$ is the feature dimensionality, and $m$ is the data size, which is very large for large-scale data), our one-pass strategy can vastly reduce the space complexity to $O(cd)$ ($c$ is the total number of classes, which is usually small). Also, the time complexity of our triplet construction strategy is $O(1)$, which is truly fast.

For fast metric updating (second issue), several studies \cite{chechik2010large,jain2009online} try to adopt a closed-form solution for accurate computation. Among them, OASIS adopts bilinear similarity learning and has a closed-form solution, while it lacks a good interpretability as Mahalanobis distance metric learning (i.e., linear projection) and the learned similarity function is asymmetric. In contrast, LogDet Exact Gradient Online (LEGO) \cite{jain2009online} attempts to learn a Mahalanobis distance and has a closed-form solution. In addition, LEGO is not required to maintain the PSD constraint by using LogDet regularization, which is time-consuming in some Mahalanobis distance metric learning methods \cite{shalev2004online,jin2009regularized,davis2007information}. However, LEGO is designed for pairwise constraints. Compared with LEGO, we developed a different Mahalanobis distance-based OML method for triplet-based constraints, named as OPML, which also has the property of closed-form solution and does not need projection steps to maintain PSD constraint. Specifically, the proposed OPML directly learns the transformation matrix $\bm{L}$ ($\bm{M}=\bm{L}^T\bm{L}$ is the symmetric PSD matrix usually learnt in Mahalanobis metric learning), such a setting does not require imposing the PSD constraint. By carefully analysing the structure of the triplets based loss and using a few fundamental properties (Lemma \ref{lemma2}, Lemma \ref{lemma3}), a closed-form solution at each step is obtained with the time complexity of $O(d^2)$. The major differences between OPML and OASIS/LEGO can be found in Table \ref{OML-methods}.

%----------- The compare of different OML methods-------------%
\begin{table}[t]\small
\centering
%\extrarowheight=0.1pt
\tabcolsep=2.5pt
\caption{The comparison of different OML methods. MA/BI denotes the Mahalanobis distance-based/bilinear similarity-based method, respectively. The last 3 columns denote the processing time (the unit is ms) per sample with different dimensions}
\vspace{5pt}
\begin{tabular}{|cccc|ccc|}
  \hline
  Method       &Type   &Constraint      &Solution            &d=21          &d=64          &d=310\\
  \hline
  \hline
  POLA         &MA     &Pair            &approximate         &$8$           &$6.2$         &$120$\\
  RDML         &MA     &Pair            &approximate         &$0.040$       &$0.043$       &$2.4$\\
  LEGO         &MA     &Pair            &closed-form         &$0.159$       &$0.472$       &$47$\\
  OASIS        &BI     &Triplet         &closed-form         &$0.029$       &$0.028$       &$9.4$\\
  SOML         &BI     &Triplet         &approximate         &$0.032$       &$0.094$       &$21$\\
  \hline
  OPML         &MA     &Triplet         &closed-form         &$\bm{0.026}$  &$\bm{0.023}$  &$\bm{1.7}$\\
  COPML        &MA     &Pair $\!\&\!$ Triplet &closed-form         &$\bm{0.027}$  &$\bm{0.024}$  &$\bm{1.7}$\\
  \hline
\end{tabular}
\label{OML-methods}
\vspace{-10pt}
\end{table}

Also, in some tasks, e.g., abnormal event detection in videos, the data is usually imbalanced: the first several samples may belong to the same class, then the triplet construction strategy will be invalid until the samples of different classes appear. We call this case as cold start case. Furthermore, to deal with the cold start issue, an extension namely COPML is developed in this paper. Specifically, COPML includes a pre-stage by constructing pairwise constraints for two adjacent samples (from the same class) to update the metric.

To summarize, compared with previous OML methods, the advantages of our work can be concluded as: First, the proposed OPML and COPML are easy to implement. Second, OPML and COPML are scalable to large datasets with a low space (i.e., $O(d)$) and time (i.e., $O(d^2)$) complexity, where $d$ is the feature dimensionality. Third, we have derived several theoretical explanations, including the difference bound between learned metrics of ones-pass and batch triplet construction strategies, the average loss bound between these two strategies and the regret bound, to guarantee the effectiveness of our methods.

The rest of this paper is organized as follows. In section \ref{related-work}, we present the related works of OML methods. Section \ref{our-method} provides the details of the proposed one-pass triplet construction strategy, OPML and COPML algorithms. In section \ref{theory}, we give the theoretical guarantee of our algorithm. The experimental results, comparisons and analysis are given in section \ref{experiments}, followed by conclusions in section \ref{conclusions}.

\section{Related Work}
\label{related-work}
Typically, all the previous online metric learning methods can be roughly classified into two categories: bilinear similarity-based and Mahalanobis distance-based. A comparison of the most related works is given in Table \ref{OML-methods} for better clarification.

In bilinear similarity-based methods, OASIS \cite{chechik2010large} is developed which is based on Passive-Aggressive algorithm \cite{crammer2006online}, aiming to learn a similarity metric for image similarity. Sparse Online Metric Learning (SOML) \cite{gao2014soml} follows a similar setting as OASIS, but learns a diagonal matrix instead of a full matrix to handle very high-dimensional cases. In order to deal with multi-modal data, an online kernel based method, namely Online Multiple Kernel Similarity (OMKS), has been proposed by Xia et al. \cite{xia2014online}. All these above methods are based on triplet constraints and they also assume that the constraints can be gained beforehand or could be randomly sampled on the entire dataset. Among them, OASIS is more relevant to the proposed OPML as both of them are Passive-Aggressive based. The differences between OASIS and the proposed OPML mainly include that OASIS learns a bilinear similarity metric (hard to interpret, and asymmetric) while OPML learns a Mahalanobis metric (good interpretability, and symmetric), OASIS randomly samples triplet constraints from the entire training set (with space complexity of $O(md)$), while OPML constructs triplet constraints in an online manner (with space complexity of $O(cd)$ and time complexity of $O(1)$), which leads the solutions of objective functions largely different.

In Mahalanobis distance-based methods, Pseudo-Metric Online Learning Algorithm (POLA) \cite{shalev2004online} is the first OML method which introduces the successive projection operation to learn the optimal metric. LEGO \cite{jain2009online} is an extended version of Information Theoretic Metric Learning-Online (ITML-Online) \cite{davis2007information}, by building the model with LogDet divergence regularization. Jin et al. \cite{jin2009regularized} presented a regularized OML method namely RDML, with a provable regret bound. Also, Kunapuli and Shavlik proposed an unified approach based on composite mirror descent named as MDML \cite{kunapuli2012mirror}. These methods are all based on pairwise constraints, and they all assume that the pairwise constraints can be obtained in advance except RDML, which exactly receives two adjacent samples as a pairwise constraints at each time. In fact, pairwise constraints based methods can be easily converted to an online manner for pair construction by using the strategy of RDML. However, in general, triplet constraints are more effective than pairwise constraints for learning a metric \cite{weinberger2009distance,chechik2010large,shaw2011learning,qian2015efficient}. In contrast, the proposed OPML is a triplet constraints based Mahalanobis distance method. Besides, the proposed method has the properties of closed-form solution and does not enquire PSD constraint, making it more efficient.

%------------------ The overflow of one-pass triplet construction ---------------%
\begin{figure}[!tbp]
      \centering
      \includegraphics[width = 0.4\textwidth]{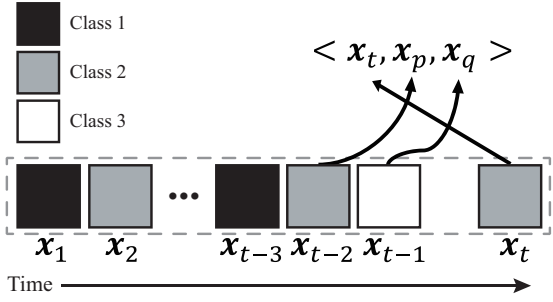}
      \vspace{-10pt}
      \caption{The illustration of one-pass triplet construction. $\bm{x}_1, \bm{x}_2, \dots, \bm{x}_t$ denote the samples at the $1$-th, $2$-th and $t$-th time steps, respectively}
      \label{oml-queue}
      \vspace{-10pt}
\end{figure}

\section{The Proposed Method}
\label{our-method}
We now first present our strategy for one-pass triplet construction, and then discuss the technical details of OPML and COPML, respectively.

\subsection{One-Pass Triplet Construction}
When dealing with large-scale data, how to fast obtain the triplets is a crucial step, since the number of triplets usually determines the tradeoff between the performance effectiveness and time efficiency. Inspired by the impressive scalability of one-pass strategies \cite{stauffer2000learning,shi2010real,gao2013one}, we proposed a one-pass triplet construction strategy, aiming to quickly obtain all the triplets in a single pass over all data. Fig. \ref{oml-queue} illustrates the main idea of the one-pass strategy of triplet construction. Formally, in an online manner, the sample at the $t$-th ($t=1,2,\dots,T$) step is denoted as $\bm{x}_t\in \mathbb{R}^d$. An array $\mathcal{H}=\{h^1,h^2,\dots,h^c\}$ is maintained, where $c$ is the total number of classes, and $h^k\in\mathbb{R}^d$ $(k=1,2,\dots,c)$ means the latest sample of the $k$-th class at the $t$-th time step.

For triplet construction, our goal is to obtain a typical triplet $\langle\bm{x}_t,\bm{x}_p,\bm{x}_q\rangle$ ($\bm{x}_t, \bm{x}_p$ and $\bm{x}_q\in\mathbb{R}^d$), which satisfies that $\bm{x}_t$ and $\bm{x}_p$ belong to the same class, while $\bm{x}_t$ and $\bm{x}_q$ belong to the different classes. Specifically, in one-pass triplet construction, at the $t$-th time step, given $\bm{x}_t$ belonging to the $k$-th class, we assign $h^k$ as $\bm{x}_p$, and randomly pick $h^{k^\prime}$ $(k^\prime=1,2,\dots,c., k^\prime\neq k)$ as $\bm{x}_q$. Then $h^k\in H$ will be replaced by $\bm{x}_t$, and thus $H$ consists of the latest sample of each class. Please note that if $\bm{x}_t$ belongs to the $(c+1)$-th class, we will assign $\bm{x}_t$ to $h^{c+1}$, and also update the number of classes as well as $H$ accordingly.

We can observe that the space complexity of this strategy is $O(cd)$. Basically, when the value of $c$ is small, the space complexity can be regarded as $O(d)$. Also, the time complexity of triplet construction at each time step is $O(1)$.

\subsection{OPML}
\label{OPML-section}
We mainly focus on the Mahalanobis distance learning here, which aims to learn a symmetric PSD matrix $\bm{M}\in\mathbb{S}_+^{d\times d}$ (cone of $d\times d$ real-valued symmetric PSD matrices) and can be formally defined as follows:
\begin{equation}\label{fun1}
  D_{\bm{M}}(\bm{x}_i,\bm{x}_j)=\sqrt{(\bm{x}_i-\bm{x}_j)^\top\bm{M}(\bm{x}_i-\bm{x}_j)},
\end{equation}
where $\bm{x}_i\in\mathbb{R}^d$ and $\bm{x}_j\in\mathbb{R}^d$ are the $i$-th and $j$-th samples, respectively. $\bm{M}$ can be mathematically decomposed as $\bm{L}^\top\bm{L}$, where $\bm{L}\in\mathbb{R}^{r\times d}$ ($r$ is the rank of $\bm{M}$) denotes the transformation matrix. Then, we can rewrite Eq. (\ref{fun1}) as:
\begin{equation}
  D_{\bm{L}}(\bm{x}_i,\bm{x}_j)=\lVert\bm{L}(\bm{x}_i-\bm{x}_j)\rVert_2^2.
\end{equation}

Our goal is to learn a transformation matrix $\bm{L}$ that satisfies the following large margin constraint:
\begin{equation}
  D_{\bm{L}}(\bm{x}_i,\bm{x}_l)>D_{\bm{L}}(\bm{x}_i,\bm{x}_j)+1, \forall \bm{x}_i,\bm{x}_j,\bm{x}_l\in\mathbb{R}^d,
\end{equation}
where $\bm{x}_i$ and $\bm{x}_j$ belong to the same class, while $\bm{x}_i$ and $\bm{x}_l$ belong to different classes. We can define the hinge loss function as below:
\begin{equation}
\mathcal{G}((\bm{x}_i,\bm{x}_j,\bm{x}_l);\bm{L})=\max\big(0,1+D_{\bm{L}}(\bm{x}_i,\bm{x}_j)-D_{\bm{L}}(\bm{x}_i,\bm{x}_l)\big).
\end{equation}

By applying the one-pass triplet construction, at the $t$-th time step, we can obtain the triplet $\langle\bm{x}_t, \bm{x}_p, \bm{x}_q\rangle$. Thus, the online optimization formulation can be defined as follows by using Passive-Aggressive algorithm \cite{crammer2006online}:
\begin{equation}
  \begin{split}
   \bm{L}_t&=\underset{\bm{L}}{\arg\min}\ \Gamma(\bm{L})\\
           &=\underset{\bm{L}}{\arg\min}\ \frac{1}{2}\lVert\bm{L}-\bm{L}_{t-1}\rVert_F^2+\frac{\gamma}{2}[1+\lVert\bm{L}(\bm{x}_t-\bm{x}_p)\rVert_2^2\\
           &-\lVert\bm{L}(\bm{x}_t-\bm{x}_q)\rVert_2^2]_+,
  \end{split}
\end{equation}
where $\gamma$ is the regularization parameter, which is set into the range of $(0,\frac{1}{4})$ as a sufficient condition to theoretically guarantee the positive definite property (see Lemma \ref{lemma2}). $\lVert\cdot\rVert_\text{F}^2$ means Frobenius norm, $\lVert\cdot\rVert_2$ denotes $\ell_2-$norm, and $[z]_+=\max(0,z)$, namely the hinge loss $\mathcal{G}((\bm{x}_t,\bm{x}_p,\bm{x}_q);\bm{L})$.

The optimal solution can be obtained when the gradient vanishes $\frac{\partial\Gamma(\bm{L})}{\partial\bm{L}}=0$, hence we have:
\begin{equation}
   \frac{\partial\Gamma(\bm{L})}{\partial\bm{L}}=\left\{
   \begin{array}{ll}
   \bm{L}-\bm{L}_{t-1}+\gamma\bm{L}\bm{A}_t=0 & [z]_+>0\\
   \\
   \bm{L}-\bm{L}_{t-1}=0 & [z]_+=0,
   \end{array}
   \right.
\end{equation}
where $\bm{A}_t=(\bm{x}_t-\bm{x}_p)(\bm{x}_t-\bm{x}_p)^\top-(\bm{x}_t-\bm{x}_q)(\bm{x}_t-\bm{x}_q)^\top\in\mathbb{R}^{d\times d}$. Since $\bm{x}_t, \bm{x}_p$ and $\bm{x}_q$ are all nonzero vectors, $\bm{x}_t\neq\bm{x}_p$ and $\bm{x}_t\neq\bm{x}_q$, the rank of $(\bm{x}_t-\bm{x}_p)(\bm{x}_t-\bm{x}_p)^\top$ and $(\bm{x}_t-\bm{x}_q)(\bm{x}_t-\bm{x}_q)^\top$ is 1. Hence, the rank of $\bm{A}_t$ is 1 or 2. When $\bm{x}_t-\bm{x}_p\neq\mu(\bm{x}_t-\bm{x}_q)$, we can get that $rank(\bm{A}_t)=2$.
%------------------- Lemma 1 Start----------------------%
\begin{lemma}
Let $\bm{M}_1, \bm{M}_2$ be two PSD matrices, and $\bm{\Omega}=\bm{M}_1-\bm{M}_2$, the eigenvalue of $\bm{\Omega}$, denoted by $\lambda(\bm{\Omega})$, satisfies the following equation:
\begin{equation}
-\lambda_\text{max}(\bm{M}_2)\le\lambda(\bm{\Omega})\le\lambda_\text{max}(\bm{M}_1),
\end{equation}
where $\lambda_\text{max}(\bm{M}_1)$ and $\lambda_\text{max}(\bm{M}_2)$ are the maximum eigenvalues of $\bm{M}_1$ and $\bm{M}_2$, respectively.
\label{lemma1}
\end{lemma}
\begin{proof}
$\forall \bm{x}\in\mathbb{R}^d$ and $\bm{x}\neq\bm{0}$, $\bm{x}^\top\bm{\Omega}\bm{x}=\bm{x}^\top(\bm{M}_1-\bm{M}_2)\bm{x}$.
Since $\bm{M}_1$ and $\bm{M}_2$ are both PSD matrices, $\bm{x}^\top\bm{M}_1\bm{x}\ge0$ and $\bm{x}^\top\bm{M}_2\bm{x}\ge0$.
Thus,
\begin{equation}\label{fun2}
-\frac{\bm{x}^\top\bm{M}_2\bm{x}}{\bm{x}^\top\bm{x}}\le\frac{\bm{x}^\top\bm{\Omega}\bm{x}}{\bm{x}^\top\bm{x}}\le\frac{\bm{x}^\top\bm{M}_1\bm{x}}{\bm{x}^\top\bm{x}}.
\end{equation}
According to Rayleigh quotient, we have $\lambda_\text{min}(\bm{\Omega})\le\frac{\bm{x}^\top\bm{\Omega}\bm{x}}{\bm{x}^\top\bm{x}}\le\lambda_\text{max}(\bm{\Omega})$. Assume $\frac{\bm{x}^\top\bm{\Omega}\bm{x}}{\bm{x}^\top\bm{x}}$ achieve its maximum $\lambda_\text{max}(\bm{\Omega})$ when $\bm{x}=\bm{e}$, i.e., $\frac{\bm{e}^\top\bm{\Omega}\bm{e}}{\bm{e}^\top\bm{e}}=\lambda_\text{max}(\bm{\Omega})$. By Eq. (\ref{fun2}), we obtain
\begin{equation}
\lambda_\text{max}(\bm{\Omega})=\frac{\bm{e}^\top\bm{\Omega}\bm{e}}{\bm{e}^\top\bm{e}}\le\frac{\bm{e}^\top\bm{M}_1\bm{e}}{\bm{e}^\top\bm{e}}\le\lambda_\text{max}(\bm{M}_1)
\end{equation}
Hence, $\lambda_\text{max}(\bm{\Omega})\le\lambda_\text{max}(\bm{M}_1)$. In the similar way, we can prove that $-\lambda_\text{max}(\bm{M}_2)\le\lambda_\text{min}(\bm{\Omega})$. Thus,
\begin{equation}
-\lambda_\text{max}(\bm{M}_2)\le\lambda_\text{min}(\bm{\Omega})\le\lambda(\bm{\Omega})\le\lambda_\text{max}(\bm{\Omega})\le\lambda_\text{max}(\bm{M}_1).
\end{equation}
\end{proof}
%---------------- Lemma 1 END-----------------%
%------------------- Lemma 2 Start----------------------%
\begin{lemma}
If $0<\gamma<\frac{1}{4}$ and samples are normalized, $\bm{I}+\gamma\bm{A}_t$ is a positive definite matrix and also it is invertible, where $\bm{I}\in\mathbb{R}^{d\times d}$ is the identity matrix.
\label{lemma2}
\end{lemma}
\begin{proof}
Let $\bm{M}_1=(\bm{x}_t-\bm{x}_p)(\bm{x}_t-\bm{x}_p)^\top$, $\bm{M}_2=(\bm{x}_t-\bm{x}_q)(\bm{x}_t-\bm{x}_q)^\top$ and $\bm{A}_t=\bm{M}_1-\bm{M}_2$. It is easy to obtain that, $\lambda_\text{max}(\bm{M}_1)=\lVert\bm{x}_t-\bm{x}_p\rVert_2^2$ and $\lambda_\text{max}(\bm{M}_2)=\lVert\bm{x}_t-\bm{x}_q\rVert_2^2$.
According to Lemma \ref{lemma1}, we can obtain that $-\lambda_\text{max}(\bm{M}_2)\le\lambda(\bm{A}_t)\le\lambda_\text{max}(\bm{M}_1)$. Thus,
\begin{equation}
1-\gamma\lambda_\text{max}(\bm{M}_2)\le\lambda(\bm{I}+\gamma\bm{A}_t)\le1+\gamma\lambda_\text{max}(\bm{M}_1).
\end{equation}
For normalized samples, namely $0\le\lVert\bm{x}_t\rVert\le1$, the ranges of $\lambda_\text{max}(\bm{M}_1)$ and $\lambda_\text{max}(\bm{M}_2)$ vary from [0,4]. Given $0<\gamma<\frac{1}{4}$, $\lambda(\bm{I}+\gamma\bm{A}_t)\ge1-\gamma\lambda_\text{max}(\bm{M}_2)>0$. Obviously, $\bm{I}+\gamma\bm{A}_t$ is a symmetric matrix. Hence, $\bm{I}+\gamma\bm{A}_t$ is a positive definite matrix, and it is invertible.
\end{proof}
%---------------- Lemma 2 END-----------------%
According to Lemma \ref{lemma2}, the optimal $\bm{L}_t$ can be updated as below:
\begin{equation}\label{fun8}
   \bm{L}_t=\left\{
   \begin{array}{ll}
   \bm{L}_{t-1}(\bm{I}+\gamma\bm{A}_t)^{-1} & [z]_+>0\\
   \\
   \bm{L}_{t-1} & [z]_+=0.
   \end{array}
   \right.
\end{equation}
It is known that the time complexity of the matrix inversion in Eq. (\ref{fun8}) is $O(d^3)$. However, the rank-2 property of $\bm{A}_t$ offers us a nice way to accelerate the speed by applying Lemma \ref{lemma3}.
%---------------- Lemma 3 Start-----------------%
\begin{lemma}
\cite{miller1981inverse} Given $\bm{G}$ and $\bm{G}+\bm{B}$ as two nonsingular matrices, and let $\bm{B}$ have rank $r>0$. Let $\bm{B}=\bm{B}_1+\dots+\bm{B}_r$, where each $\bm{B}_k$ has rank 1, also let $\bm{C}_{k+1}=\bm{G}+\bm{B}_1+\dots+\bm{B}_k$ is nonsingular for $k=1,\dots,r$. If $\bm{C}_1=\bm{G}$, then
\begin{equation}
   \begin{split}
      (\bm{G}+\bm{B})^{-1}=\bm{C}_r^{-1}-g_r\bm{C}_r^{-1}\bm{B}_r\bm{C}_r^{-1},
   \end{split}
\end{equation}
where,
\begin{equation}
   \begin{split}
     \bm{C}_{r+1}^{-1}&=\bm{C}_r^{-1}-g_r\bm{C}_r^{-1}\bm{B}_r\bm{C}_r^{-1}\\
      g_r&=\frac{1}{1+tr(\bm{C}_r^{-1}\bm{B}_r)}.
   \end{split}
\end{equation}
\label{lemma3}
\end{lemma}
%---------------- Lemma 3 END --------------------%
%---------------- Theorem 1 Start-----------------%
\begin{theorem}
If $\bm{A}_t$ is a rank-2 matrix and samples are normalized, then
  \begin{equation}\label{fun11}
     (\bm{I}+\gamma\bm{A}_t)^{-1}=\bm{I}-\frac{1}{\eta+\beta}[\eta\gamma\bm{A}_t-(\gamma\bm{A}_t)^2],
  \end{equation}
where $\eta=1+tr(\gamma\bm{A}_t)$, $\beta=\frac{1}{2}[(tr(\gamma\bm{A}_t))^2-tr(\gamma\bm{A}_t)^2].$
\label{theorem1}
\end{theorem}

\begin{proof}
According to Lemma \ref{lemma3}, we set $\bm{G}=\bm{I}$, $\bm{B}=\gamma\bm{A}_t$, and rewrite $\bm{B}=\bm{B}_1+\bm{B}_2$, where $\bm{B}_1=\gamma(\bm{x}_t-\bm{x}_p)(\bm{x}_t-\bm{x}_p)^\top$, $\bm{B}_2=-\gamma(\bm{x}_t-\bm{x}_q)(\bm{x}_t-\bm{x}_q)^\top$. It is obvious that the rank of $\bm{B}_1$ and $\bm{B}_2$ is 1. Utilizing the Lemma \ref{lemma3}, we can obtain the Theorem \ref{theorem1}.
\end{proof}
%---------------- Theorem 1 END-----------------%

By using Theorem \ref{theorem1}, plugging Eq. (\ref{fun11}) back into the first term of Eq. (\ref{fun8}), we can obtain
\begin{equation}\label{fun12}
  \begin{split}
      \bm{L}_t&=\bm{L}_{t-1}-\frac{\eta\gamma}{\eta+\beta}(\bm{L}_{t-1}\bm{a}\bm{a}^\top-\bm{L}_{t-1}\bm{b}\bm{b}^\top)\\
            &+\frac{\gamma^2}{\eta+\beta}[(\bm{a}^\top\bm{a})\bm{L}_{t-1}\bm{a}\bm{a}^\top-(\bm{a}^\top\bm{b})\bm{L}_{t-1}\bm{a}\bm{b}^\top\\
            &-(\bm{b}^\top\bm{a})\bm{L}_{t-1}\bm{b}\bm{a}^\top+(\bm{b}^\top\bm{b})\bm{L}_{t-1}\bm{b}\bm{b}^\top],
  \end{split}
\end{equation}
where $\bm{a}=\bm{x}_t-\bm{x}_p$, and $\bm{b}=\bm{x}_t-\bm{x}_q$.

\textbf{Complexity:} The time complexity of calculating $\eta$ and $\beta$ in Eq. (\ref{fun12}) is $O(d^2)$. In addition, $\bm{a}^\top\bm{a}$, $\bm{a}^\top\bm{b}$, $\bm{b}^\top\bm{a}$ and $\bm{b}^\top\bm{b}$ are scalars, for all of which time complexity is $O(d)$. Also, the time complexity of calculating $\bm{L}_{t-1}\bm{a}\bm{a}^\top$, $\bm{L}_{t-1}\bm{b}\bm{b}^\top$, $\bm{L}_{t-1}\bm{a}\bm{b}^\top$ and $\bm{L}_{t-1}\bm{b}\bm{a}^\top$ is $O(d^2)$ respectively. Therefore, the time complexity of Eq. (\ref{fun12}) is still $O(d^2)$. Now, we give the pseudo-code of OPML in Algorithm \ref{alg1}.
%------------------ OPML Algorithm Start--------------------%
\vspace{-6pt}
\begin{algorithm}[H]
\caption{OPML}
\renewcommand{\algorithmicrequire}{\textbf{Input:}}
\renewcommand{\algorithmicensure}{\textbf{Output:}}
\label{alg1}
\begin{algorithmic}[1]
\REQUIRE $(\bm{x}_t,y_t)|_{t=1}^T, \gamma$.
\ENSURE $\bm{L}$.
\STATE $\bm{L}_0 \leftarrow \bm{I}$.
\FOR{$t=1,2,\dots,T$}
\STATE $\langle\bm{x}_t, \bm{x}_p, \bm{x}_q\rangle \leftarrow$ one-pass triplet construction
\IF{$\mathcal{G}((\bm{x}_t,\bm{x}_p,\bm{x}_q);\bm{L}_{t-1})\leqslant0$}
\STATE $\bm{L}_t=\bm{L}_{t-1}.$
\ELSE
\STATE $\bm{L}_t \leftarrow$ solution by Eq. (\ref{fun12})
\ENDIF
\ENDFOR
\end{algorithmic}
\vspace{-3pt}
\end{algorithm}
%------------------ OPML Algorithm END--------------------%

\subsection{Extended OPML to Cold Start Case}
In practice, there is a case that the first several available samples belong to the same class, which is called as a cold start case. When cold start happens, since the triplet cannot be constructed, OPML will discard all these initial samples. To address this issue, we extend the proposed OPML to an enhanced version, namely COPML, which includes an additional pre-stage before calling OPML. Specifically, in the pre-stage, if the triplet cannot be constructed (i.e., the samples coming from different classes are not available), the metric $\bm{L}$ can only be updated based on the samples from the same class, which usually adopts the pairwise constraint for updating. Typically, pairwise constraint is mathematically set to $\langle\bm{x}_t, \bm{x}_{t+1}, y_{t,t+1}^\ast\rangle$, where $y_{t,t+1}^\ast=1$ if two adjacent samples $\bm{x}_t\in\mathbb{R}^d$ and $\bm{x}_{t+1}\in\mathbb{R}^d$ share the same class, and $y_{t,t+1}^\ast=-1$ otherwise. Actually, here we only need to consider the case when $y_{t,t+1}^\ast=1$, because we can update the metric by calling OPML if $y_{t,t+1}^\ast=-1$ (the new coming sample belongs to a different class). After the pre-stage, OPML can be sequentially adopted for the following learning process.

Formally, in the pre-stage when only the pairwise constraint $\langle\bm{x}_t, \bm{x}_{t+1}, y_{t,t+1}^\ast\rangle$ can be used, the online optimization formulation is formulated as follows:
\begin{equation}
  \begin{split}
  \bm{L}_t&=\underset{\bm{L}}{\arg\min}\ \Gamma(\bm{L})\\
          &=\underset{\bm{L}}{\arg\min}\ \frac{1}{2}\lVert\bm{L}-\bm{L}_{t-1}\rVert_\text{F}^2+\frac{\gamma_1}{2}y_{t,t+1}^\ast\lVert\bm{L}(\bm{x}_t-\bm{x}_{t+1})\rVert_2^2,
  \end{split}
\end{equation}
where $\gamma_1>0$ is the regularization parameter. The optimal solution can be obtained when the gradient vanishes $\frac{\partial\Gamma(\bm{L})}{\partial\bm{L}}=0$, hence
\begin{equation}
   \begin{split}
    \frac{\partial\Gamma(\bm{L})}{\partial\bm{L}}=\bm{L}-\bm{L}_{t-1}+\gamma_1 y_{t,t+1}^\ast\bm{L}\bm{\Lambda}_t=0,
   \end{split}
\end{equation}
where $\bm{\Lambda}_t=(\bm{x}_t-\bm{x}_{t+1})(\bm{x}_t-\bm{x}_{t+1})^\top$. It is obvious that $\bm{\Lambda}_t$ is a rank-1 PSD matrix for $\bm{x}_t\neq\bm{x}_{t+1}$. Then we can get that $\bm{I}+\gamma_1 y_{t,t+1}^\ast\bm{\Lambda}_t$ is a symmetric  positive definite matrix, which is invertible, when $y_{t,t+1}^\ast=1$. Then, the optimal $\bm{L}_t$ can be obtained as
\begin{equation}\label{fun16}
   \begin{split}
      \bm{L}_t=\bm{L}_{t-1}(\bm{I}+\gamma_1\bm{\Lambda}_t)^{-1}.
   \end{split}
\end{equation}
By using the Sherman-Morrison formula, Eq. (\ref{fun16}) can be equivalently rewritten as follows:
\begin{equation}\label{fun17}
      \bm{L}_t=\bm{L}_{t-1}-\frac{\gamma_1\bm{L}_{t-1}(\bm{x}_t-\bm{x}_{t+1})(\bm{x}_t-\bm{x}_{t+1})^\top}{1+\gamma_1 (\bm{x}_t-\bm{x}_{t+1})^\top(\bm{x}_t-\bm{x}_{t+1})},
\end{equation}
where we can observe that the time complexity of Eq. (\ref{fun17}) is $O(d^2)$ too. Algorithm \ref{alg2} shows the pseudo-code of COPML.
%----------------- COPML Algorithm Start -----------------%
\vspace{-6pt}
\begin{algorithm}[H]
\caption{COPML}
\renewcommand{\algorithmicrequire}{\textbf{Input:}}
\renewcommand{\algorithmicensure}{\textbf{Output:}}
\label{alg2}
\begin{algorithmic}[1]
\REQUIRE $(\bm{x}_t,y_t)|_{t=1}^T, \gamma_1, \gamma_2$.
\ENSURE $\bm{L}$.
\STATE $\bm{L}_0 \leftarrow \bm{I}, c\leftarrow 0$.

\FOR{$t=1,2,\dots,T$}
\STATE Maintain $H=\{h^1,h^2,\cdots,h^c\}$
\IF{$y_t=c=1$}
\STATE $\langle\bm{x}_t,\bm{x}_{t+1}, y_{t,t+1}^\ast\rangle\leftarrow$ adjacent two samples
\STATE $\bm{L}_t=\bm{L}_{t-1}-\frac{\gamma_1\bm{L}_{t-1}(\bm{x}_t-\bm{x}_{t+1})(\bm{x}_t-\bm{x}_{t+1})^\top}{1+\gamma_1(\bm{x}_t-\bm{x}_{t+1})^\top(\bm{x}_t-\bm{x}_{t+1})}.$
\ELSIF{$1\le y_t\le c$ and $c\ge2$ }
\STATE $L\leftarrow$ call OPML$(\bm{x}_t,y_t,\gamma_2)$
\ELSE
\STATE $h^{c+1}\leftarrow \bm{x}_t$
\STATE $c\leftarrow c+1$
\ENDIF
\ENDFOR
\end{algorithmic}
\vspace{-3pt}
\end{algorithm}
%------------------ COPML Algorithm END--------------------%

\section{Theoretical Guarantee}
\label{theory}
The following theorems guarantee the effectiveness of our methods. Theorem \ref{theorem2} shows that the difference of learned metric between one-pass triplet construction strategy and batch triplet construction strategy is bounded. Note that, for a fair comparison, the batch triplet construction strategy here is considered in an online manner, that is to say, for each sample $\bm{x}_t$ at the $t$-th time step, all past samples are stored to construct a triplet with $\bm{x}_t$ (i.e., each triplet contains this $\bm{x}_t$). Theorem \ref{theorem3} also tries to explain that the one-pass triplet construction strategy can approximate the batch triplet construction, but from another perspective. Moreover, a regret bound has been proved for the proposed OPML algorithm, which can be found in Theorem \ref{theorem4}. All details of the proofs for the theorems are provided in the appendix.

%---------------- Theorem 2 Start-----------------%
\begin{theorem}
Let $\bm{L}_t$ be the solution output by OPML based on the one-pass triplet construction strategy at the $t$-th time step. Let $\bm{L}_t^\ast$ be the solution output by OPML with the batch triplet construction strategy at the $t$-th time step. Assuming that $\|\bm{x}\|\le R$ (for all samples), $\|\bm{L}_t\|_F\le U$ and $\|\bm{L}_t^\ast\|_F\le U$, the bound of the difference between these two matrices is
\begin{equation}
     \|\bm{L}_t-\bm{L}_t^\ast\|_F\le U\cdot\Big\|\sum_{i=1}^{C_N}\bm{B}_i+\sum_{i=1,j=1,i<j}^{C_N}\bm{B}_i\bm{B}_j+\cdots+\prod_{i=1}^{C_N}\bm{B}_i\Big\|_F,
\end{equation}
where $\|\bm{B}\|_F\!\le\!32\Big|\frac{\gamma^2}{\eta+\beta}\Big|R^4\!+\!4\sqrt{2}\Big|\frac{\eta\gamma}{\eta+\beta}\Big|R^2$ ( for all $\bm{B}_i, \bm{B}_j,\cdots$), $\gamma\in(0,\frac{1}{4})$, $\eta\in(1\!-\!\frac{5}{4}R^2,1\!+\!\frac{5}{4}R^2)$, and $\beta\in(-R^4, \frac{25}{32}R^4)$.
\label{theorem2}
\end{theorem}
%---------------- Theorem 2 End-----------------%

%---------------- Theorem 3 Start---------------%
\begin{theorem}
Let $\langle\bm{x}_t, \bm{x}_p, \bm{x}_q\rangle$ be the triplet constructed by the proposed one-pass triplet construction strategy at the $t$-th time step. Let $\{\langle\bm{x}_t, \bm{x}_{p_i}, \bm{x}_{q_i}\rangle\}|_{i=1}^C$ be the triplet set constructed by the batch triplet construction strategy at the $t$-th time step. Assuming $\|\bm{x}\|_2\le R$ (for all samples), $\|\bm{L}\|_F\le U$, $\|\bm{L}^\ast\|_F\le U$ and the angle $\theta$ between two samples coming from the same class is very small after the transformation of $\bm{L}$ or $\bm{L}^\ast$ (i.e., $\cos\theta=\alpha, \alpha\ge0$ and $\alpha$ is close to 1), while $\theta$ is very large otherwise (i.e., $\cos\theta=-\xi, \xi\ge0$ and $\xi$ is close to 1). Then the average loss bound between these two strategies at the $t$-th time step is
  \begin{equation}\label{fun-the2}
    \Psi_1-\Psi_2\le2(\alpha+\xi+1)R^2U^2,
  \end{equation}
where $\Psi_1$ denotes the average loss generated by the one-pass triplet construction strategy, and $\Psi_2$ refers to the average loss of the batch construction strategy.
\label{theorem3}
\end{theorem}
%---------------- Theorem 3 End-----------------%

%---------------- Theorem 4 Start---------------%
\begin{theorem}
Let $\langle\bm{x}_1, \bm{x}_{p_1}, \bm{x}_{q_1}\rangle,\dots,\langle\bm{x}_T, \bm{x}_{p_T}, \bm{x}_{q_T}\rangle$ be a sequence of triplets constructed by the proposed one-pass strategy. Let $\bm{L}_t|_{t=1}^T$ be the solution output by OPML at the $t$-th time step, and $\bm{L}_\ast$ be the optimal offline solution. Assuming $\|\bm{x}\|_2\le R$ (for all samples), $\|\bm{L}\|_F\le U$, $\|\bm{L}_\ast\|_F\le U$ and the angle $\theta$ between two samples coming from the same class is small after the transformation of $\bm{L}$ or $\bm{L}^\ast$ (i.e., $\cos\theta=\alpha, \alpha\ge0$ and $\alpha$ is close to 1), while $\theta$ is large otherwise (i.e., $\cos\theta=-\xi, \xi\ge0$ and $\xi$ is close to 1). Then the regret bound is
\begin{equation}\label{fun-the2}
     R(\bm{L}_\ast, T)\le2T(\alpha+\xi+1)R^2U^2
\end{equation}
\label{theorem4}
\end{theorem}
%---------------- Theorem 4 End-----------------%

\section{Experiments}
\label{experiments}
To verify the effectiveness of our methods, we evaluate OPML and COPML on three typical tasks, including (1) UCI data classification, (2) face verification, and (3) abnormal event detection in videos. Also, an additional experiment is conducted to validate the robustness of COPML when the cold start issue happens.
%-------------------- Error Rates & run time start---------------------%
\begin{table*}[!htb]%\scriptsize%\footnotesize%\small
\centering
%\vspace{-10pt}
%\renewcommand{\arraystretch}{0.83}
%\extrarowheight=0.05pt
\tabcolsep=8pt
\caption{Error rates (mean$\pm$std. deviation) of a k-NN (k=5) classifier on the UCI datasets. $p-$values of student's t-test are calculated between other methods and our methods. $\bullet/\circ$ indicates OPML performs statistically better/worse than the respective method according to the $p-$values. The statistics of win/tie/loss is also included. The value in the bracket means the corresponding total processing time in second. 0.00 denotes the value is very small ($<0.005$). Abbreviations: Sam, sample; Dim, dimensionality; C, classes}
\vspace{2pt}
\begin{tabular}{|L{36pt}C{15pt}C{14pt}C{13pt}|c|c|c|ccccccccccc|}
\hline
Data         &Sam      &Dim   &C                           &Euclidean                             &Mahalanobis                         &LMNN\\
\hline
lsvt         &126      &310   &2                           &$0.234\pm0.056\bullet$                &$0.238\pm0.046\bullet$              &$0.196\pm0.044\ \ \ (46.08)$\\

iris         &150      &4     &3                           &$0.050\pm0.023\ \ $                   &$0.075\pm0.025\bullet$              &$0.037\pm0.017\circ(\ \ 1.55)$\\

wine         &178      &13    &3                           &$0.044\pm0.020\ \ $                   &$0.046\pm0.023\ \ $                 &$0.030\pm0.016\circ(\ \ 3.18)$ \\

glass        &214      &9     &7                           &$0.336\pm0.036\ \ $                   &$0.349\pm0.037\ \ $                 &$0.341\pm0.038\ \ \ (\ \ 4.52)$\\

spect        &267      &22    &2                           &$0.327\pm0.035\ \ $                   &$0.348\pm0.034\bullet$              &$0.336\pm0.037\ \ \ (\ \ 2.48)$\\

ionosphere   &351      &34    &2                           &$0.172\pm0.019\bullet$                &$0.165\pm0.016\ \ $                 &$0.131\pm0.020\circ(\ \ 3.81)$\\

balance      &625      &4     &3                           &$0.146\pm0.014\bullet$                &$0.133\pm0.017\bullet$              &$0.124\pm0.014\circ(\ \ 1.42)$\\

breast       &683      &9     &2                           &$0.034\pm0.008\ \ $                   &$0.034\pm0.007\ \ $                 &$0.033\pm0.008\ \ \ (\ \ 1.26)$ \\

pima         &768      &8     &2                           &$0.273\pm0.018\bullet$                &$0.271\pm0.018\bullet$              &$0.272\pm0.017\bullet(\ \ 1.12)$\\

segment      &2310     &19    &7                           &$0.067\pm0.006\bullet$                &$0.101\pm0.008\bullet$              &$0.047\pm0.006\circ(\ \ 5.59)$\\

waveform     &5000     &21    &3                           &$0.187\pm0.006\circ$                  &$0.158\pm0.006\circ$                &$0.182\pm0.060\circ(\ \ 6.04)$\\

optdigits    &5620     &64    &10                          &$0.026\pm0.003\bullet$                &$0.039\pm0.003\bullet$              &$0.014\pm0.002\circ(31.21)$\\
\hline
\multicolumn{4}{|c|}{\textbf{win/tie/loss}}                &\textbf{6/5/1}                        &\textbf{7/4/1}                      &\textbf{1/4/7} \\
\hline
\hline
\multicolumn{4}{|c|}{ITML}                                 &OASIS                                 &RDML                                &POLA\\
\hline
\multicolumn{4}{|c|}{$0.175\pm0.040\circ(142.3)$}          &$0.205\pm0.043\bullet(0.84)$          &$0.230\pm0.053\bullet\ (0.17)$      &$0.157\pm0.036\circ(\ \ 9.17)$\\

\multicolumn{4}{|c|}{$0.034\pm0.016\circ(11.16)$}          &$0.274\pm0.050\bullet(0.10)$          &$0.077\pm0.027\bullet(\bm{0.00})$   &$0.030\pm0.016\circ(\ \ 0.62)$\\

\multicolumn{4}{|c|}{$0.035\pm0.019\circ(14.15)$}          &$0.019\pm0.014\circ(0.08)$            &$0.039\pm0.019\ \ \ (\bm{0.00})$    &$0.028\pm0.018\circ(\ \ 2.22)$ \\

\multicolumn{4}{|c|}{$0.358\pm0.042\bullet(30.94)$}        &$0.485\pm0.057\bullet(0.12)$          &$0.349\pm0.035\ \ \ (\bm{0.00})$    &$0.395\pm0.042\ \ \ (\ \ 3.86)$\\

\multicolumn{4}{|c|}{$0.337\pm0.041\ \ \ (\ \ 3.44)$}      &$0.364\pm0.050\bullet(0.13)$          &$0.343\pm0.032\bullet(\bm{0.01})$   &$0.323\pm0.040\bullet(19.22)$\\

\multicolumn{4}{|c|}{$0.139\pm0.025\circ(\ \ 3.65)$}       &$0.124\pm0.033\circ(0.12)$            &$0.154\pm0.016\circ\ (0.02)$        &$0.147\pm0.020\circ(17.65)$\\

\multicolumn{4}{|c|}{$0.104\pm0.019\circ(15.67)$}          &$0.126\pm0.009\ \ \ (0.10)$           &$0.118\pm0.012\circ(\bm{0.01})$     &$0.156\pm0.047\bullet(\ \ 7.31)$\\

\multicolumn{4}{|c|}{$0.035\pm0.008\ \ \ (\ \ 2.68)$}      &$0.043\pm0.021\bullet(0.09)$          &$0.033\pm0.007\ \ \ (\bm{0.01})$    &$0.038\pm0.010\bullet(\ \ 5.13)$ \\

\multicolumn{4}{|c|}{$0.279\pm0.022\bullet(\ \ 3.16)$}     &$0.346\pm0.053\bullet(0.12)$          &$0.269\pm0.019\ \ \ (\bm{0.01})$    &$0.275\pm0.021\bullet(12.40)$\\

\multicolumn{4}{|c|}{$0.050\pm0.008\circ(38.50)$}          &$0.343\pm0.067\bullet(0.11)$          &$0.082\pm0.006\bullet(\bm{0.02})$   &$0.057\pm0.011\ \ \ (22.89)$\\

\multicolumn{4}{|c|}{$0.187\pm0.008\circ(18.57)$}          &$0.357\pm0.039\bullet(0.13)$          &$0.186\pm0.010\circ\ (0.12)$        &$0.250\pm0.030\bullet(28.38)$\\

\multicolumn{4}{|c|}{$0.028\pm0.006\bullet(122.4)$}        &$0.077\pm0.009\bullet(0.15)$          &$0.028\pm0.003\bullet\ (0.13)$      &$0.023\pm0.003\bullet(19.76)$\\
\hline
\multicolumn{4}{|c|}{\textbf{3/2/7}}                       &\textbf{9/1/2}                        &\textbf{5/4/3}                      &\textbf{6/2/4} \\
\hline
\hline
\multicolumn{4}{|c|}{LEGO}                                 &SOML                                  &\textbf{OPML}                       &\textbf{COPML}\\
\hline
\multicolumn{4}{|c|}{$0.239\pm0.050\bullet(\ \ 2.33)$}     &$0.223\pm0.057\bullet(\ \ 2.23)$      &$0.189\pm0.048\ \ (\bm{0.07})$      &$0.189\pm0.047\ \ (\bm{0.07})$\\

\multicolumn{4}{|c|}{$0.050\pm0.021\ \ \ (\ \ 0.14)$}      &$0.287\pm0.080\bullet(\ \ 0.69)$      &$0.049\pm0.023\ \ (\bm{0.00})$      &$0.048\pm0.023\ \ (\bm{0.00})$\\

\multicolumn{4}{|c|}{$0.031\pm0.020\circ(\ \ 0.23)$}       &$0.169\pm0.093\bullet(\ \ 0.92)$      &$0.042\pm0.020\ \ (\bm{0.00})$      &$0.041\pm0.019\ \ (\bm{0.00})$\\

\multicolumn{4}{|c|}{$0.390\pm0.034\bullet(\ \ 0.31)$}     &$0.557\pm0.117\bullet(\ \ 1.16)$      &$0.339\pm0.032\ \ (\bm{0.00})$      &$0.341\pm0.033\ \ (\bm{0.00})$\\

\multicolumn{4}{|c|}{$0.311\pm0.038\circ(\ \ 0.61)$}       &$0.393\pm0.084\bullet(\ \ 1.45)$      &$0.326\pm0.034\ \ (\bm{0.01})$      &$0.327\pm0.035\ \ (\bm{0.01})$\\

\multicolumn{4}{|c|}{$0.154\pm0.020\circ(\ \ 1.09)$}       &$0.362\pm0.116\bullet(\ \ 2.18)$      &$0.161\pm0.019\ \ (\bm{0.01})$      &$0.163\pm0.021\ \ (\bm{0.01})$\\

\multicolumn{4}{|c|}{$0.118\pm0.011\circ(\ \ 1.40)$}       &$0.378\pm0.077\bullet(\ \ 4.78)$      &$0.129\pm0.012\ \ (\bm{0.01})$      &$0.129\pm0.014\ \ (\bm{0.01})$\\

\multicolumn{4}{|c|}{$0.035\pm0.008\bullet(\ \ 1.64)$}     &$0.054\pm0.040\bullet(\ \ 5.19)$      &$0.032\pm0.008\ \ (\bm{0.01})$      &$0.032\pm0.007\ \ (\bm{0.01})$\\

\multicolumn{4}{|c|}{$0.266\pm0.019\ \ \ (\ \ 1.96)$}      &$0.353\pm0.060\bullet(\ \ 5.89)$      &$0.266\pm0.017\ \ (\bm{0.01})$      &$0.265\pm0.018\ \ (0.02)$\\

\multicolumn{4}{|c|}{$0.040\pm0.006\circ(14.50)$}          &$0.541\pm0.092\bullet(19.26)$         &$0.059\pm0.006\ \ (0.03)$           &$0.059\pm0.006\ \ (0.03)$\\

\multicolumn{4}{|c|}{$0.233\pm0.007\bullet(\ \ 4.60)$}     &$0.368\pm0.035\bullet(49.30)$         &$0.224\pm0.009\ \ (\bm{0.08})$      &$0.225\pm0.010\ \ (0.09)$\\

\multicolumn{4}{|c|}{$0.022\pm0.003\bullet(\ \ 6.97)$}     &$0.239\pm0.079\bullet(79.10)$         &$0.019\pm0.003\ \ (\bm{0.09})$      &$0.019\pm0.003\ \ (0.10)$\\
\hline
\multicolumn{4}{|c|}{\textbf{5/2/5}}                       &\textbf{12/0/0}                       &                                    &\\
\hline
\end{tabular}
\label{error-runtime}
%\vspace{-10pt}
\end{table*}
%-------------------- Error Rates & run time END---------------------%

%--------- The number of constraints ----------
\begin{figure*}[!htbp]
\centering
\subfigure{
            \includegraphics[width=0.23\textwidth]{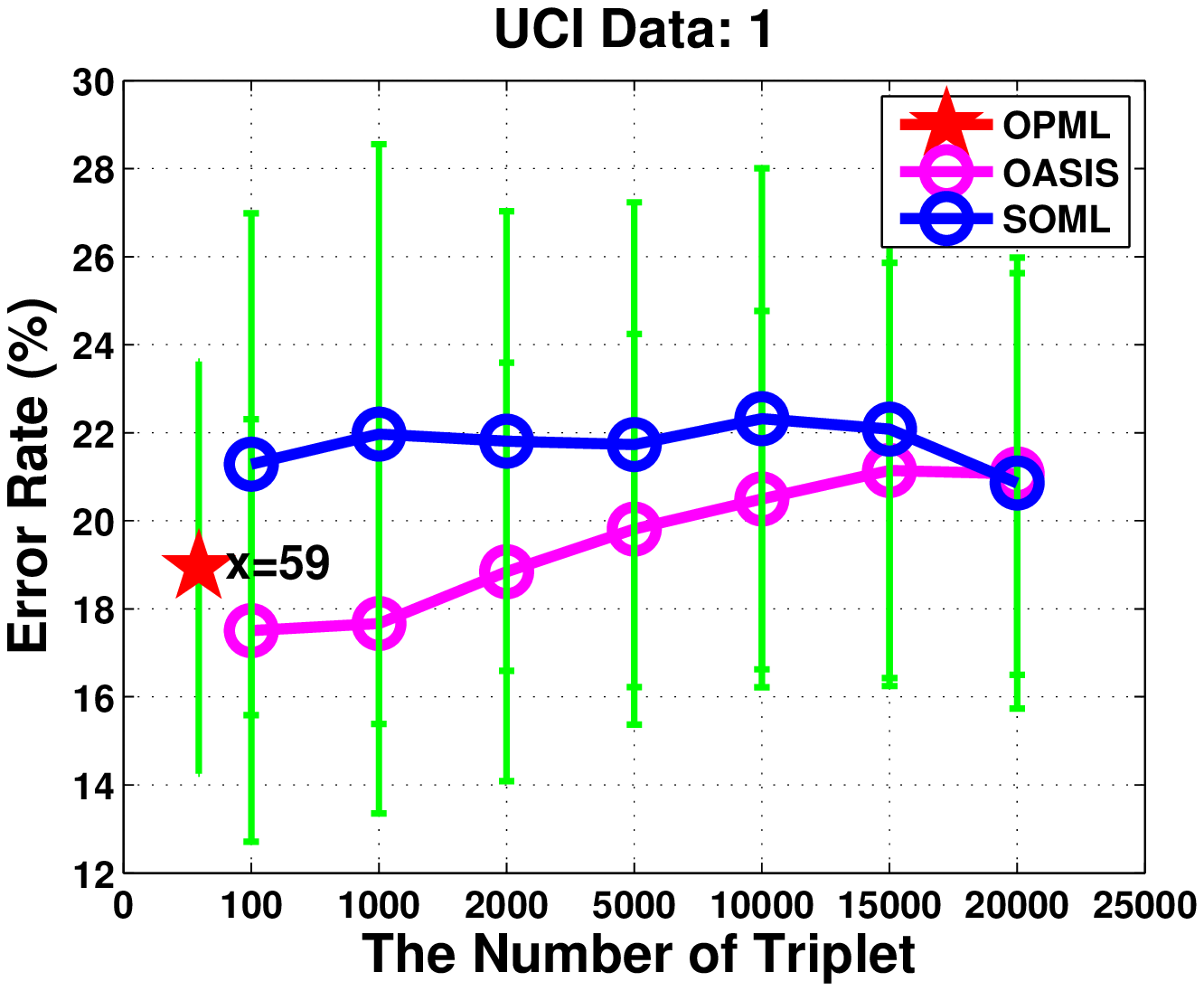}}
            %\vspace{-2pt}
\subfigure{
            \includegraphics[width=0.23\textwidth]{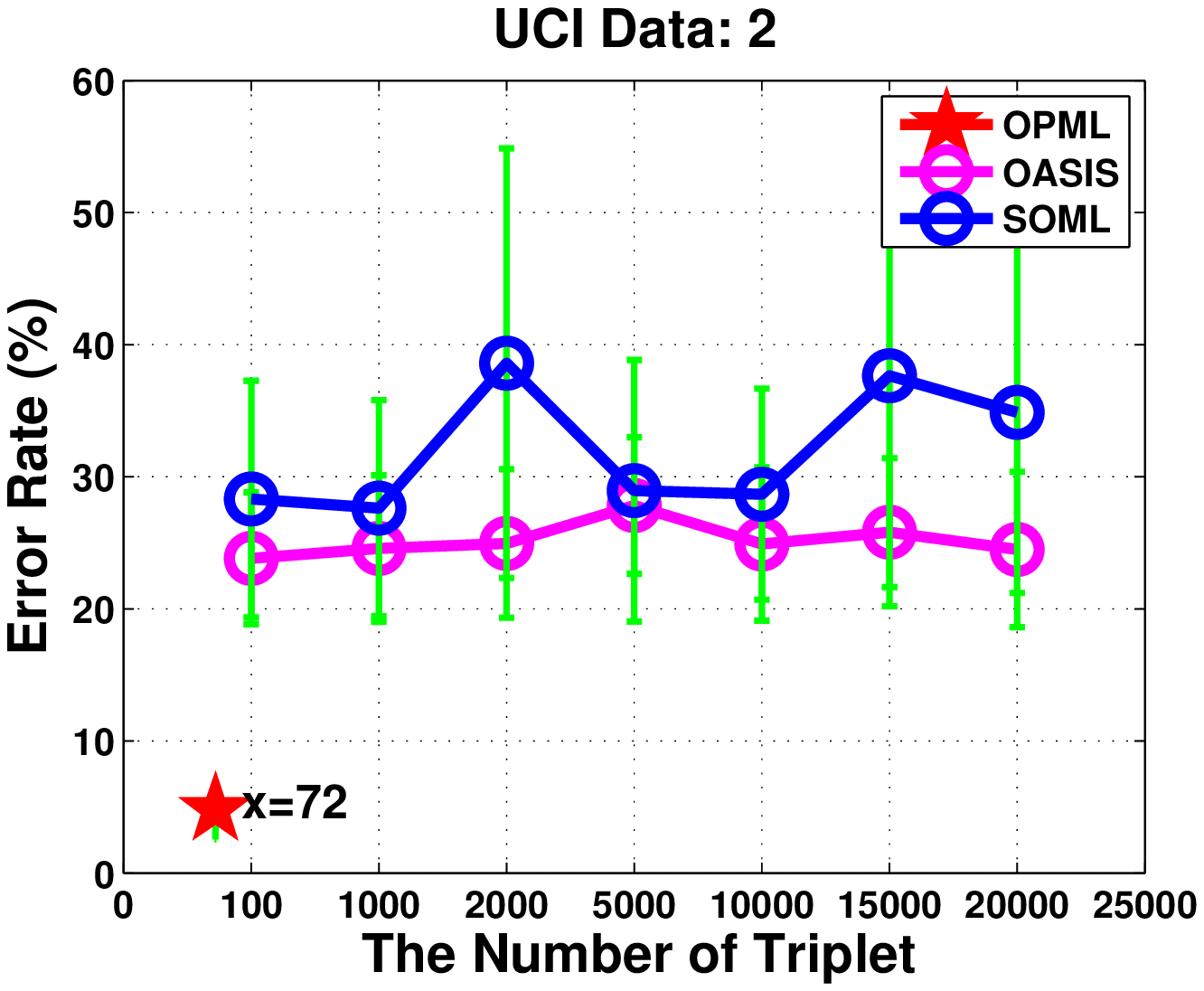}}
            %\vspace{-2pt}
\subfigure{
            \includegraphics[width=0.23\textwidth]{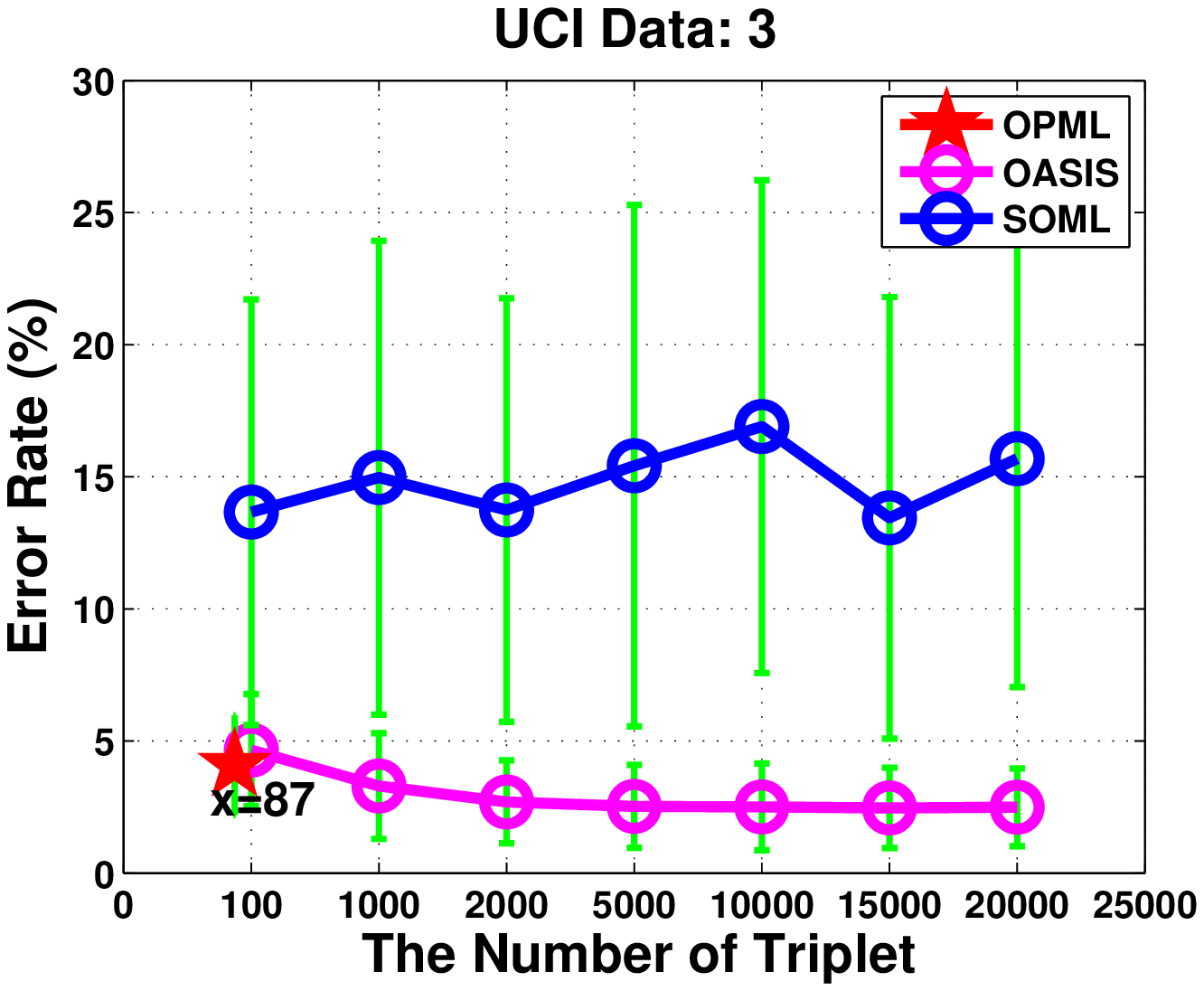}}
            %\vspace{-2pt}
\subfigure{
            \includegraphics[width=0.23\textwidth]{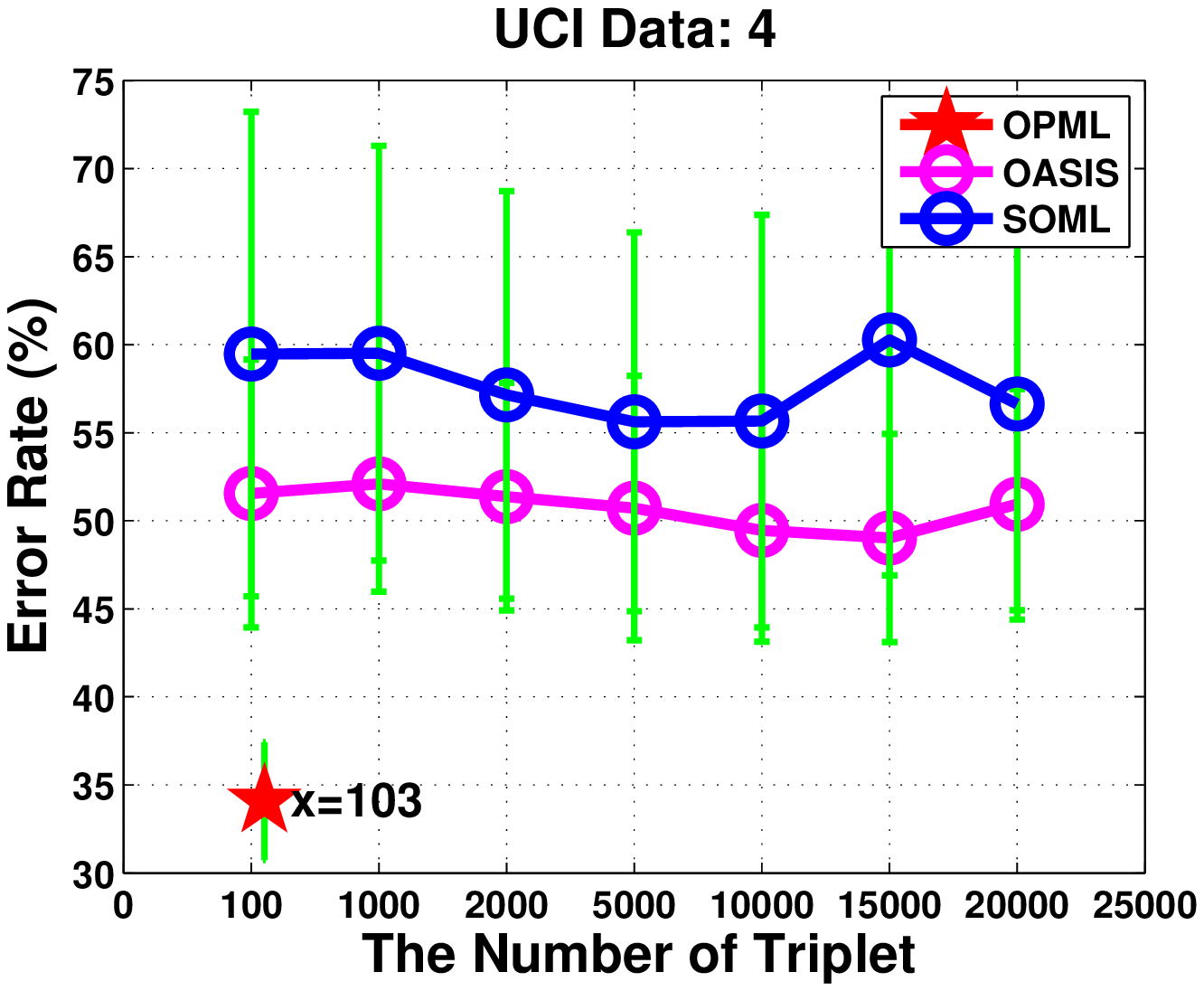}}
            %\vspace{-2pt}
\subfigure{
            \includegraphics[width=0.23\textwidth]{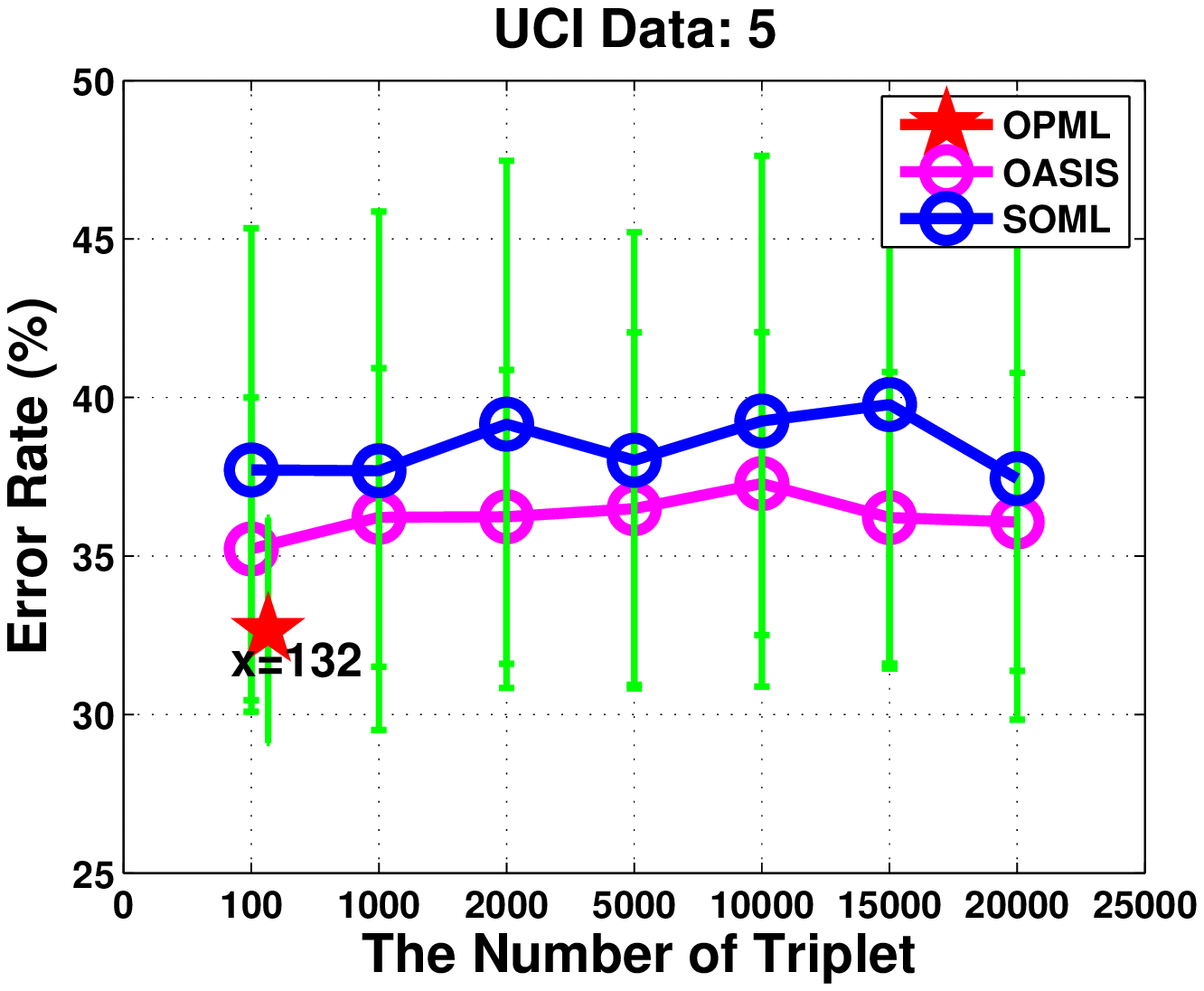}}
            %\vspace{-2pt}
\subfigure{
            \includegraphics[width=0.23\textwidth]{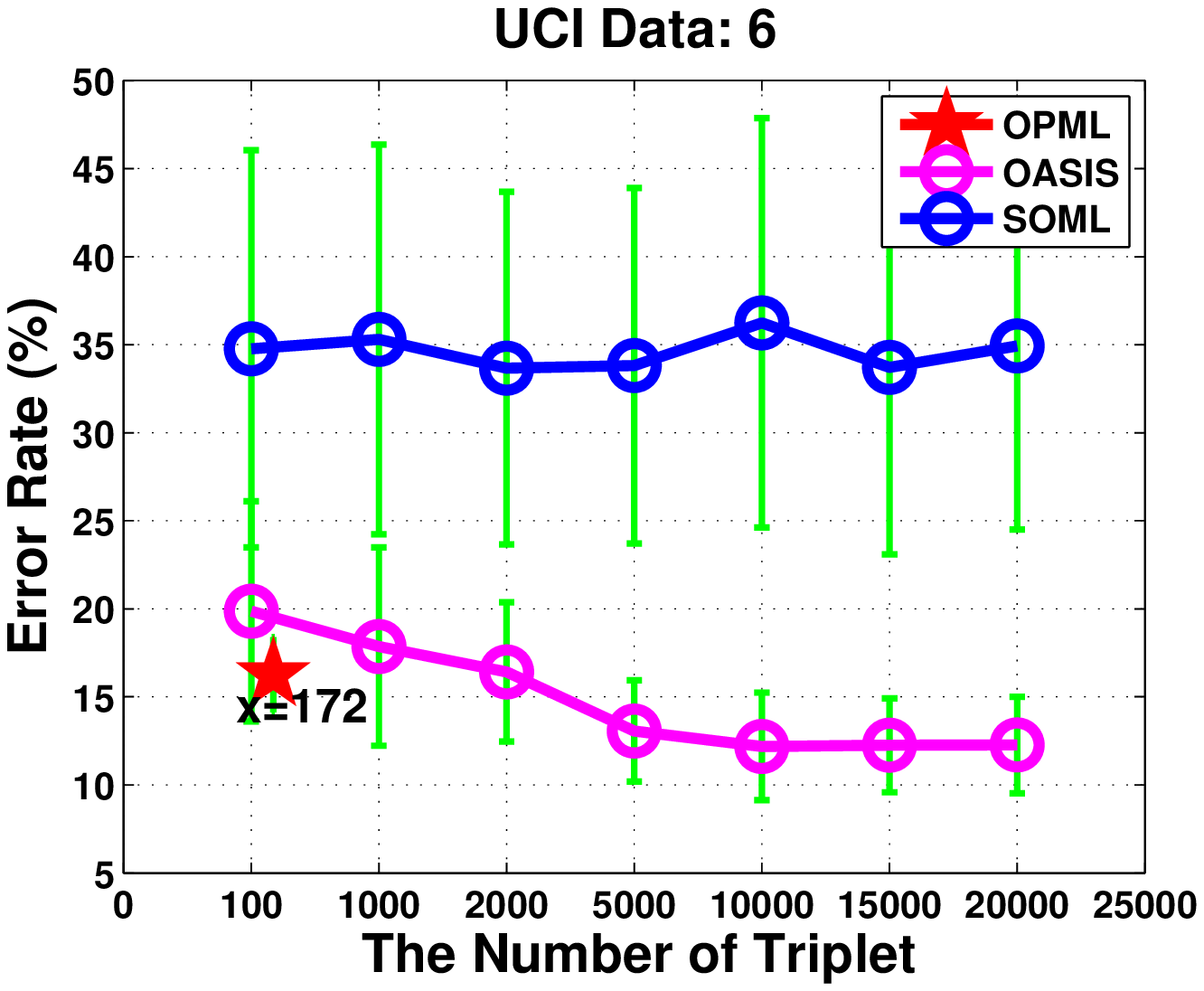}}
            %\vspace{-2pt}
\subfigure{
            \includegraphics[width=0.23\textwidth]{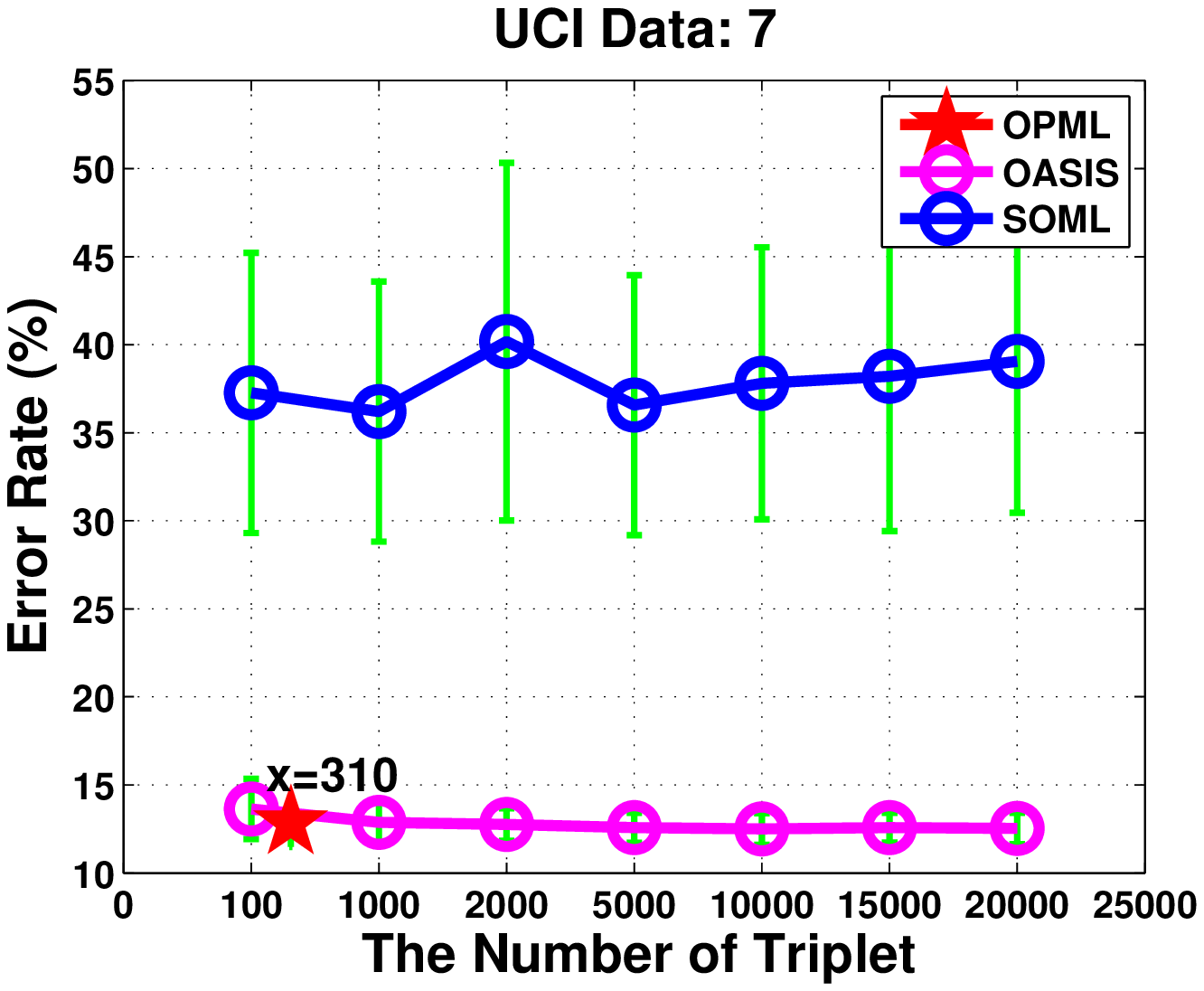}}
            %\vspace{-2pt}
\subfigure{
            \includegraphics[width=0.23\textwidth]{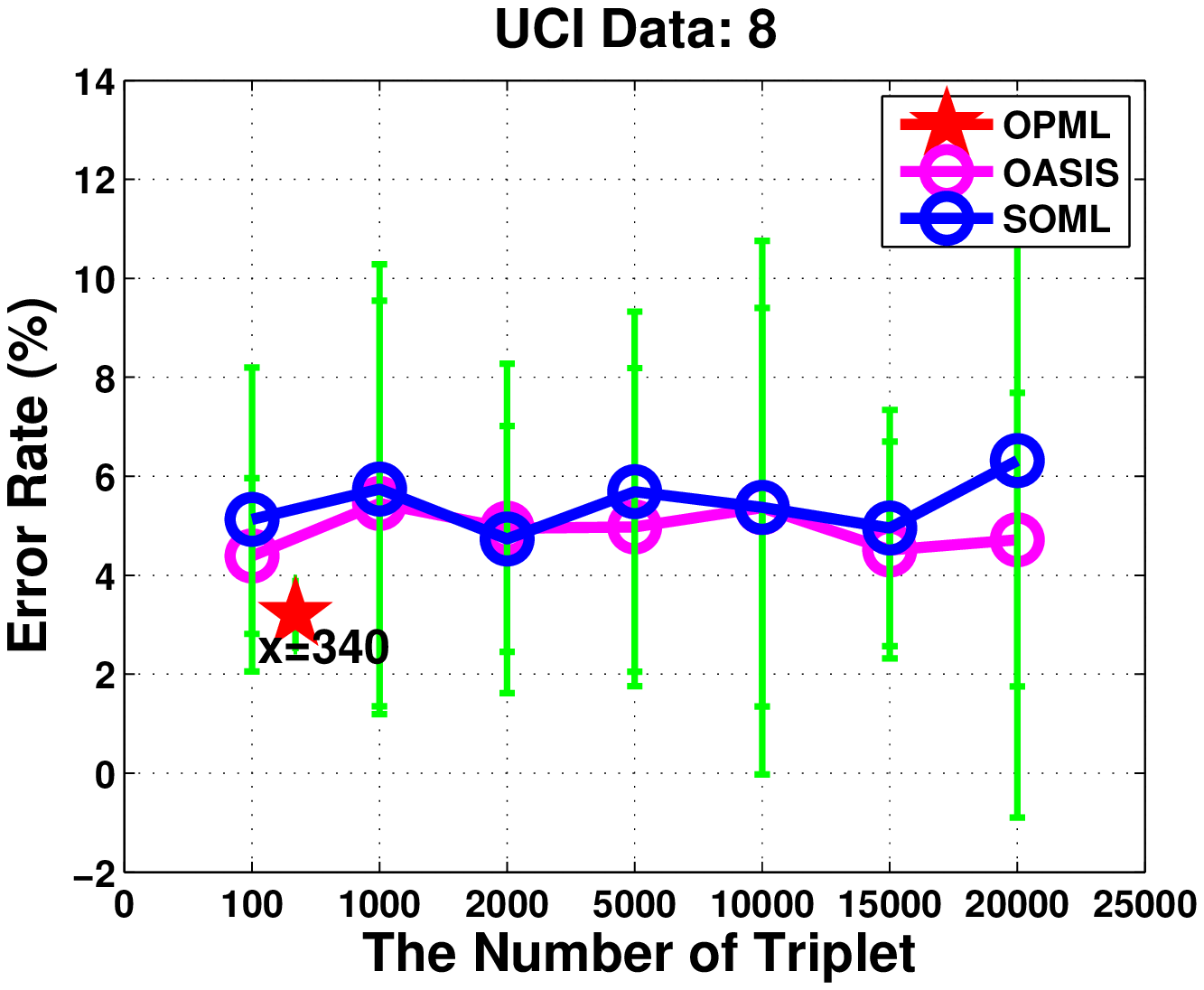}}
            %\vspace{-2pt}
\subfigure{
            \includegraphics[width=0.23\textwidth]{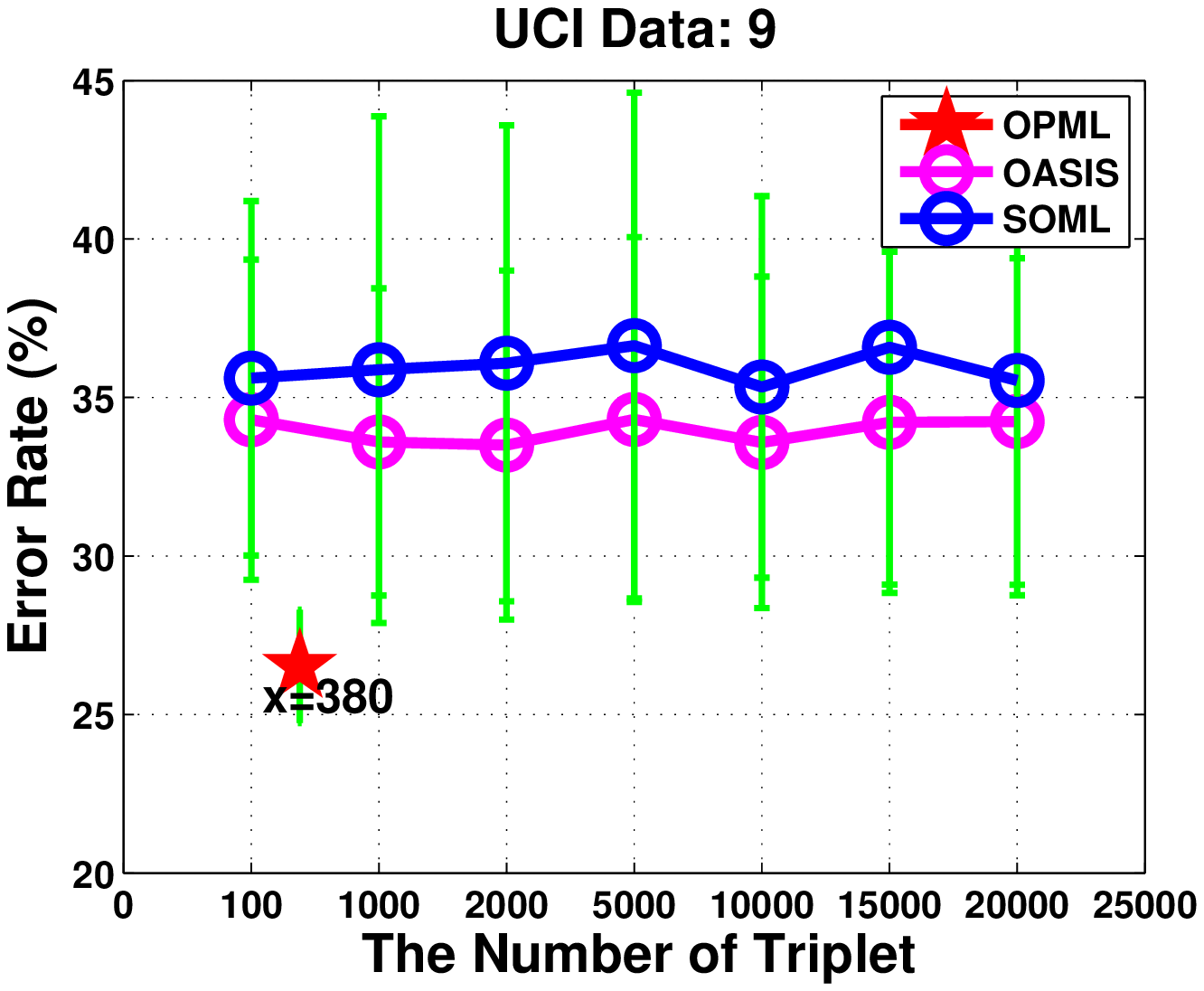}}
            %\vspace{-2pt}
\subfigure{
            \includegraphics[width=0.23\textwidth]{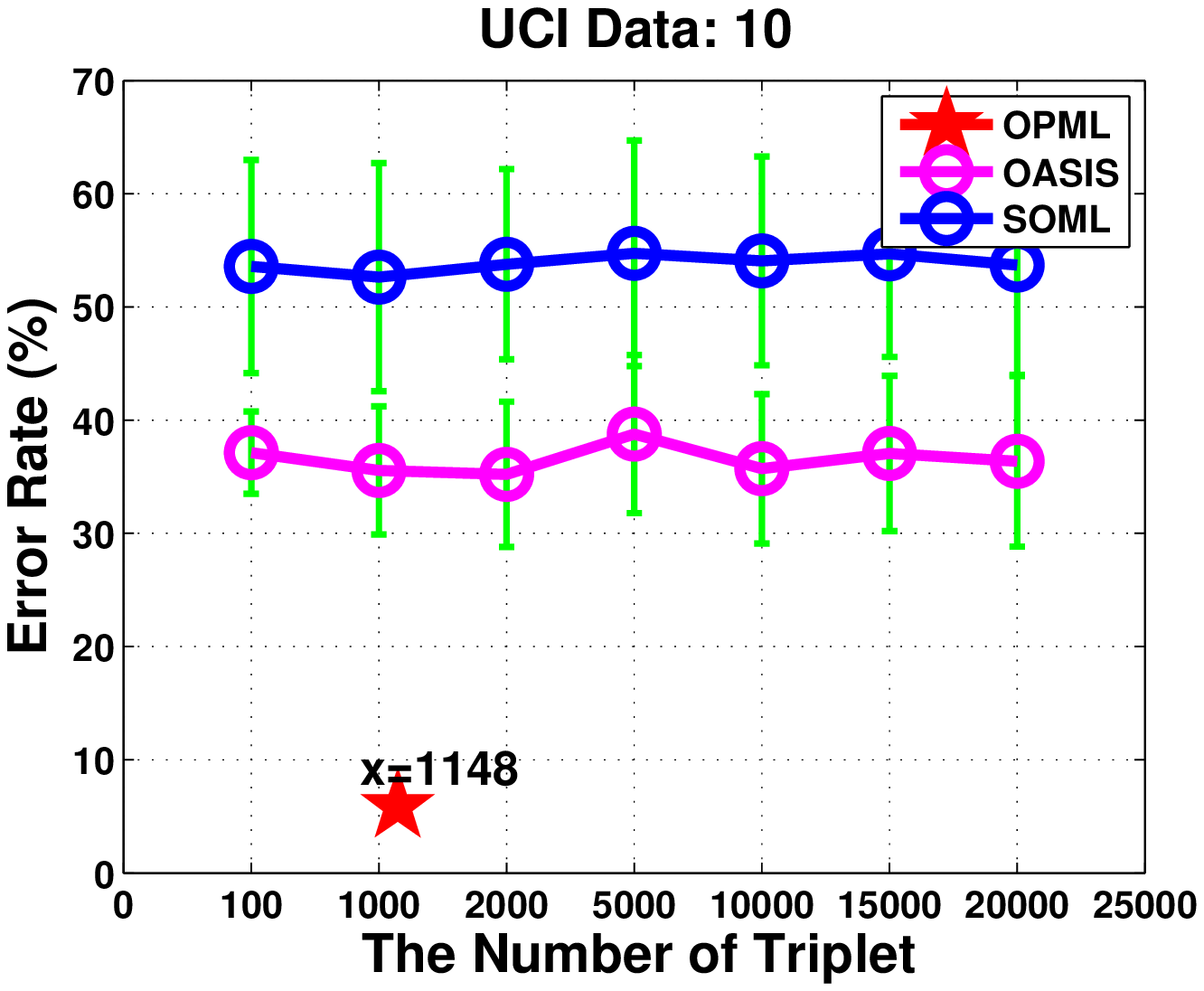}}
            %\vspace{-2pt}
\subfigure{
            \includegraphics[width=0.23\textwidth]{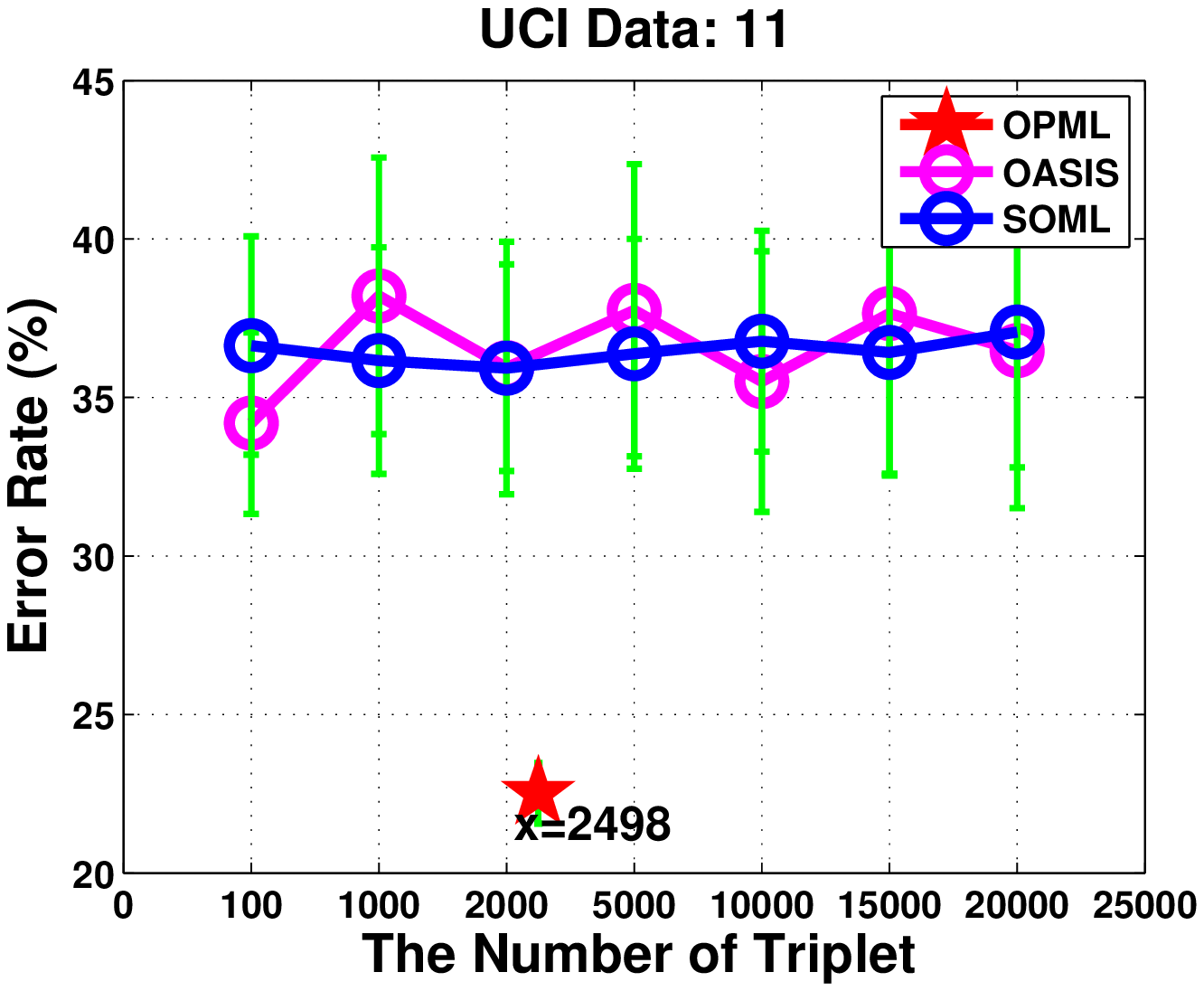}}
            %\vspace{-2pt}
\subfigure{
            \includegraphics[width=0.23\textwidth]{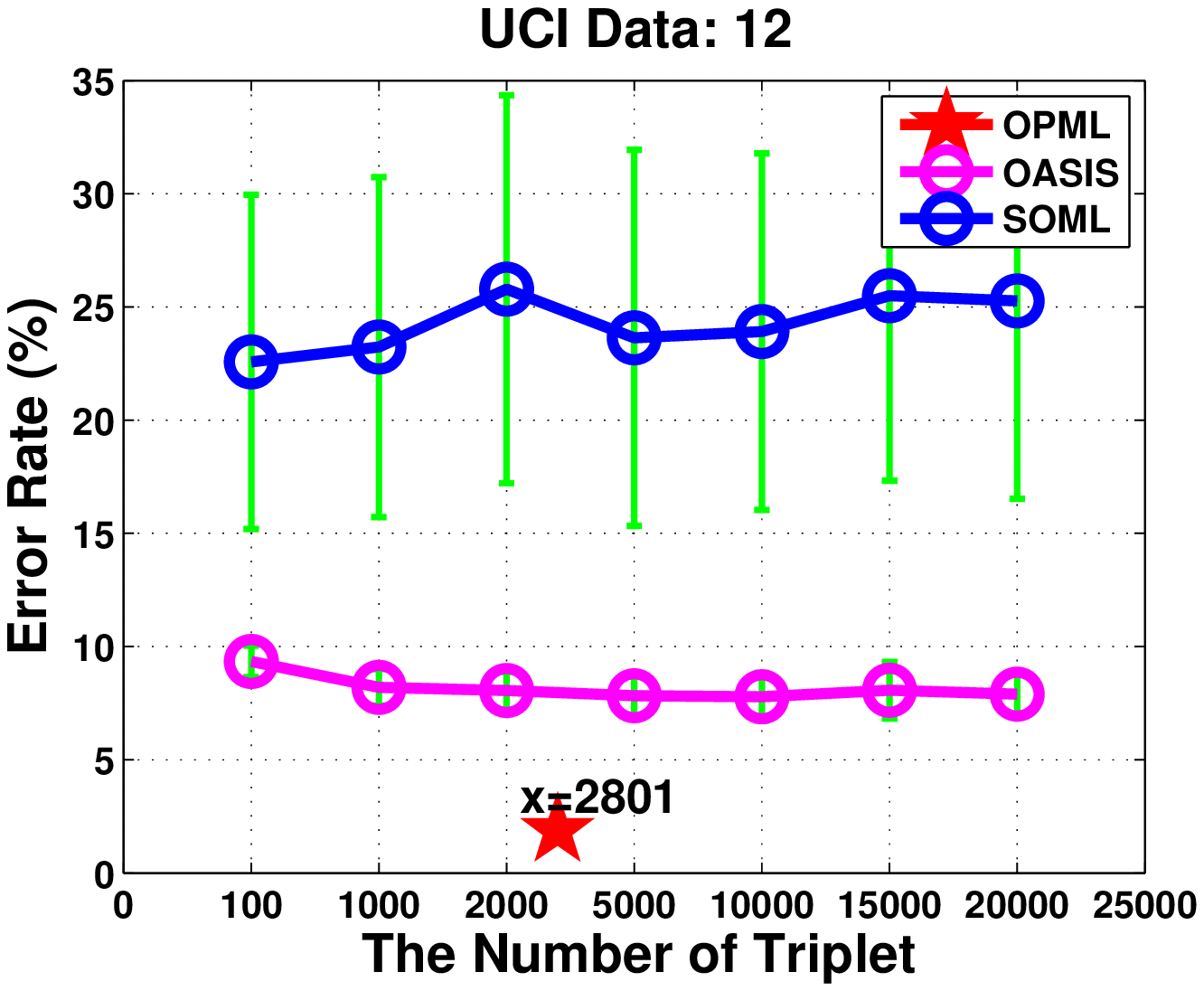}}
%\vspace{-5pt}
\caption{Error rates of different methods with different numbers of triplets on twelve UCI datasets (the number of triplets in OPML is a constant)}
\label{fig-scale}
\end{figure*}
\subsection{UCI Data Classification}

We introduce twelve datasets from the UCI repository for evaluation. The $k$-NN classifier is employed, since it is widely-used for classification with only one parameter. The detailed information of these datasets is presented in Table \ref{error-runtime}. All these twelve datasets are normalized by Z-score. Also, for each dataset, $50\%$ samples are randomly picked for training while the rest is used for testing. We adopt the error rate as the evaluation criterion, and to reduce the influence coming from the random partition, all the classification results are averaged over 100 individual runs.

To make an extensive comparison, we introduce several state-of-the-art methods, including batch metric learning and OML methods. Specifically, batch metric learning methods include: (1) \textbf{Euclidean} distance metric (\textbf{Eucli} for short); (2) \textbf{ Mahalanobis} distance metric (\textbf{Maha} for short); (3) \textbf{LMNN} (Large Margin Nearest Neighbor) \cite{weinberger2009distance}; (4) \textbf{ITML} \cite{davis2007information}. OML methods include: (1) \textbf{OASIS} \cite{chechik2010large}; (2) \textbf{RDML} \cite{jin2009regularized}; (3) \textbf{POLA} \cite{shalev2004online}; (4) \textbf{LEGO} \cite{jain2009online}; (5) \textbf{SOML-TG} (\textbf{SOML} for short) \cite{gao2014soml}. The implementation of LMNN, ITML and OASIS was provided by the authors in their respective papers, while the rest methods were implemented by ourselves. The parameters of these methods were selected by cross-validation, except LMNN and ITML using the default settings. Since the pairwise or triplet constraints of POLA, LEGO and SOML need to be constructed in advance, we randomly sample 10000 constraints for these three methods (same setting as LEGO \cite{jain2009online}). The error rates of the proposed methods and competitive methods are presented in Table \ref{error-runtime}.

Moreover, the $p$-values of student's t-test were calculated to check statistical significance. Also, the statistics of win/tie/loss is reported according to the obtained $p$-values (see Table \ref{error-runtime}). It is observed that (1) the performance of our methods is comparable to LEGO, and slightly better than other OML methods; (2) the performance of our methods is close to batch metric learning methods, e.g., LMNN and ITML, and better than Euclidean and Mahalanobis; (3) our methods are faster than other OML methods except comparable with RDML, since instead of constructing triplets, RDML only requires the pairwise constraint by receiving a pair of samples in each time.

To illustrate the performance with different numbers of triplet constraints on the learning of metric, we vary the numbers of triplet constraints as (100, 1000, 2000, 5000, 10000, 15000, 20000) for OASIS and SOML (see Fig. \ref{fig-scale}). Since the number of triplet constraint in OPML is a constant by using one-pass triplet construction, we can find that, OPML can achieve better performance by using fewer triplet constraints (except on UCI data 1, 3, 6).

\subsection{Face Verification: PubFig}
\label{section_face}
For face verification, we first evaluate our methods on the Public Figures Face Database (PubFig) \cite{kumar2009attribute}. PubFig dataset consists of two subsets: Development Set (7650 images of 60 individuals) and Evaluation Set (28954 images of 140 individuals). Following \cite{kumar2009attribute}, we use the development set to develop all these methods, including parameters tuning, while the evaluation set is used for performance evaluation. The goal of face verification in PubFig is to determine whether a pair of face images belong to the same person. Please note that, images coming from the same person will be regarded as belonging to the same class. For all subsets, 10-fold cross validation is adopted to conduct the experiments, and each fold is disjoint by identity (i.e., one person will not appear in both the training and testing set). For testing each fold (with rest 9 folds used for training), we randomly sample 10000 pairs (5000 intra- and 5000 extra-personal pairs) for testing. Thus, the total number of pairs is $10^5$. In each training phase, we also randomly select 10000 pairwise or triplet constraints for LEGO, POLA and SOML as the same settings on the UCI datasets.

For sufficient and fair comparison, we use two forms of features (i.e., attribute features and deep features) to evaluate the performance of all algorithms, respectively. Attribute features (73-dimension) provided  by Kumar et al. \cite{kumar2009attribute} are 'high-level' features describing nameable attributes such as gender, race, age, hair etc., of a face image. For deep features, we use a VGG-Face model \cite{Parkhi15} to extract a 4096-dimensional feature for each face image which has been aligned and cropped.  For easier handling, the 4096-dimensional feature is reduced to a 54-dimensional feature by Principal Component Analysis (PCA) algorithm.

For each testing pair, we first calculate the distance (similarity) between them by the learned metric obtained from respective methods. Then, all the distances (similarities) are normalized into the range $[0,1]$. Receiver Operating Characteristic (ROC) curves are provided in Fig. \ref{roc-Figure}, with the corresponding AUC (Area under ROC) values calculated. It can be observed that OPML and COPML can obtain superior results compared with the state-of-the-art online/batch metric learning methods. Moreover, although the deep feature already has a strong representation ability, our proposed methods can still slightly improve the performance.

%--------------- Roc-Curve in Face Verification -------------------
\begin{figure}[!htbp]
     %\vspace{-10pt}
     \centering
     \includegraphics[width = 0.24\textwidth]{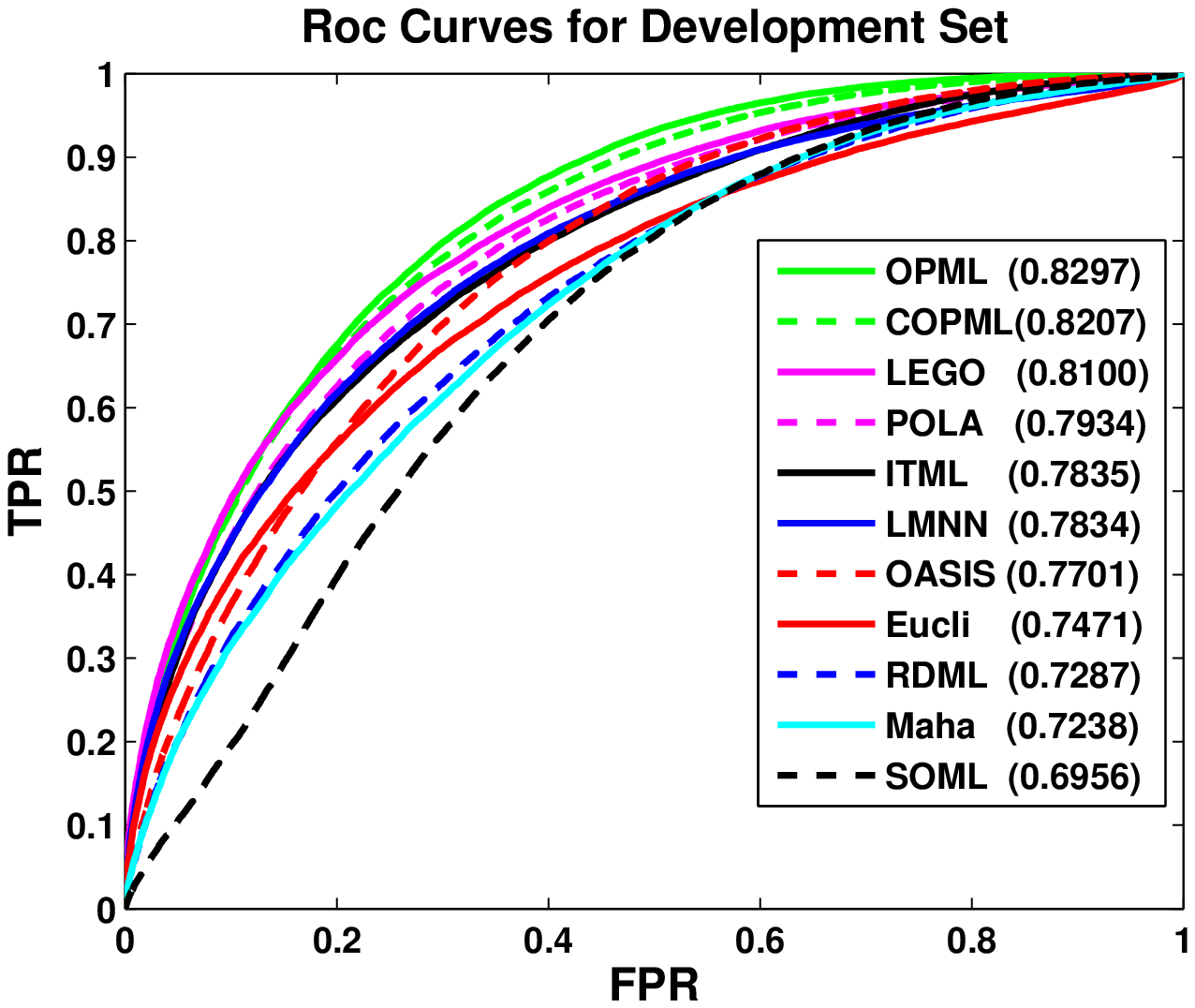}
     \includegraphics[width = 0.24\textwidth]{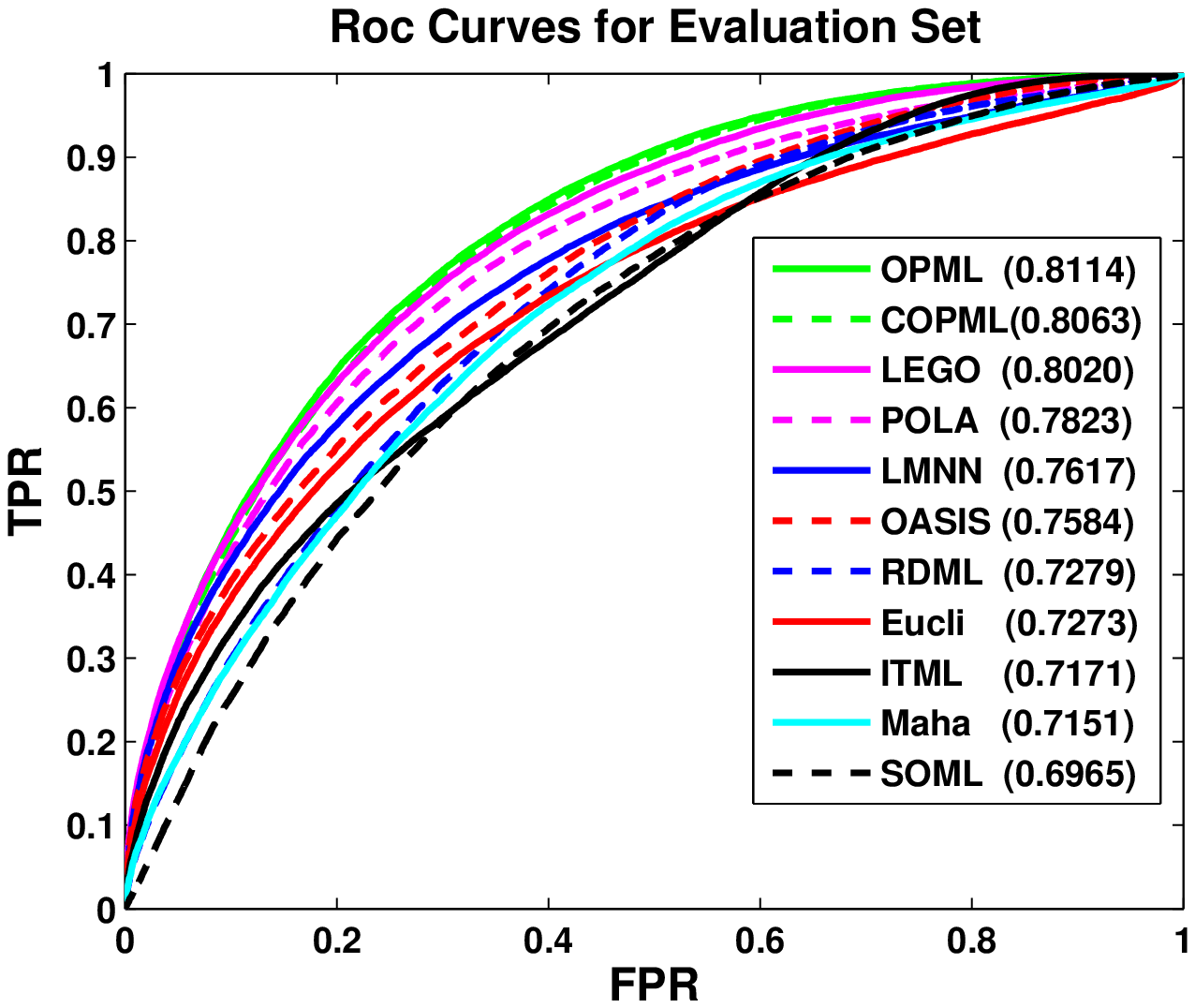}
     \includegraphics[width = 0.24\textwidth]{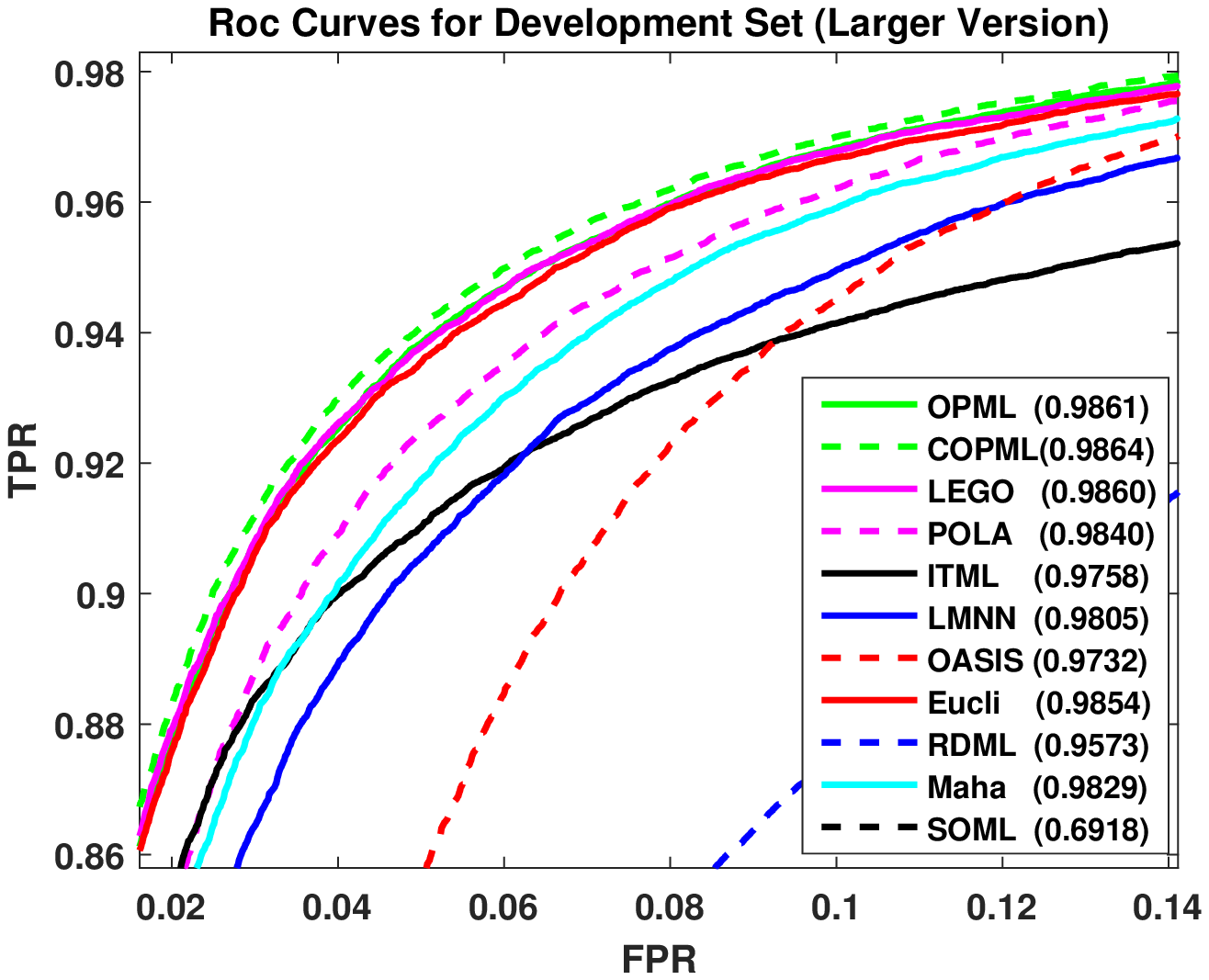}
     \includegraphics[width = 0.24\textwidth]{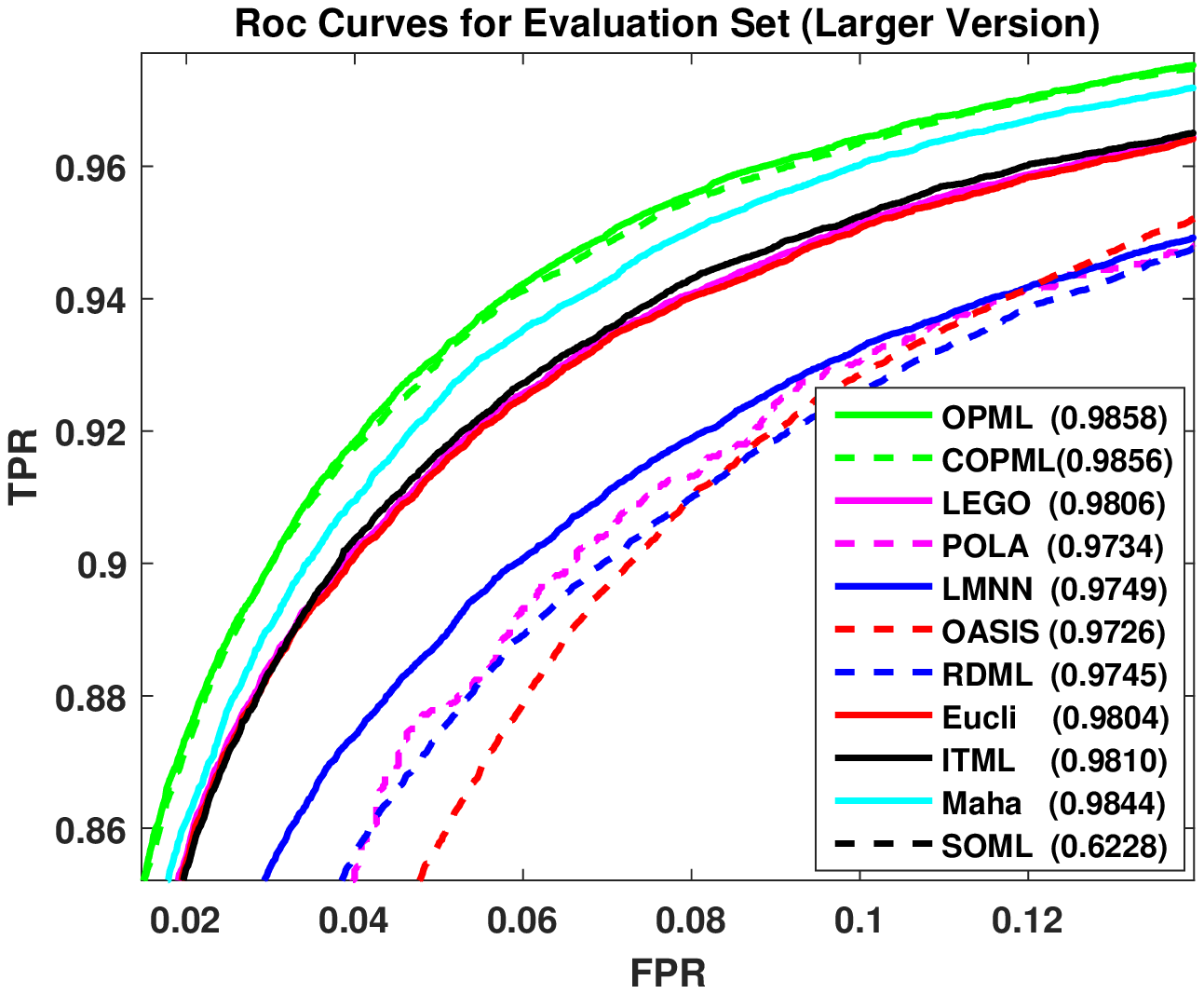}
     \caption{ROC Curves of development set (left column) and evaluation set (right column) on the PubFig dataset. first row: attribute features; second row: deep features. AUC value of each method is presented in bracket}
     \label{roc-Figure}
     %\vspace{-20pt}
\end{figure}

%--------------- Roc-Curve in Face Verification -------------------
\begin{figure}[!htbp]
     %\vspace{-10pt}
     \centering
     \subfigure{ \includegraphics[width = 0.23\textwidth]{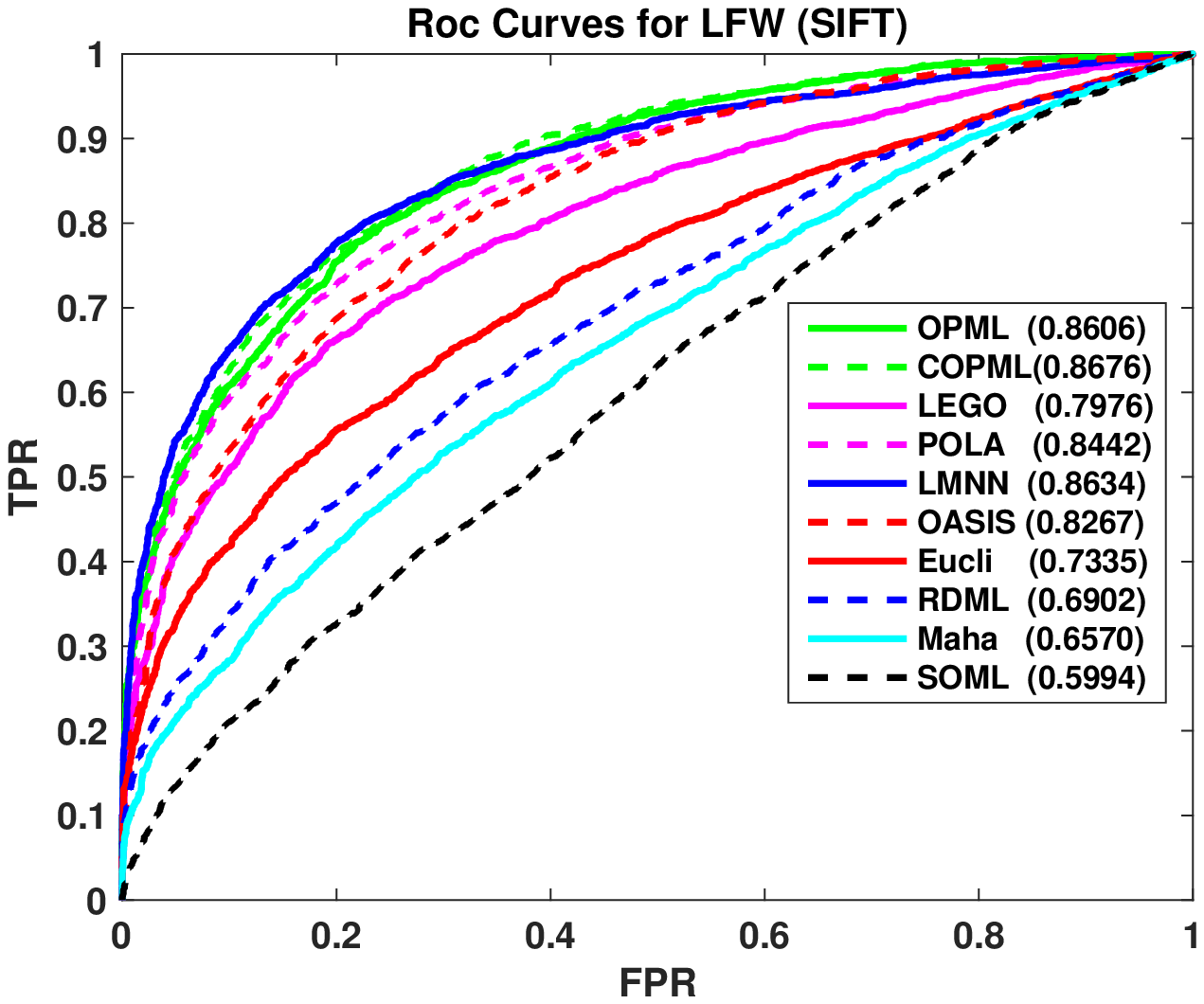} }
     \subfigure{ \includegraphics[width = 0.23\textwidth]{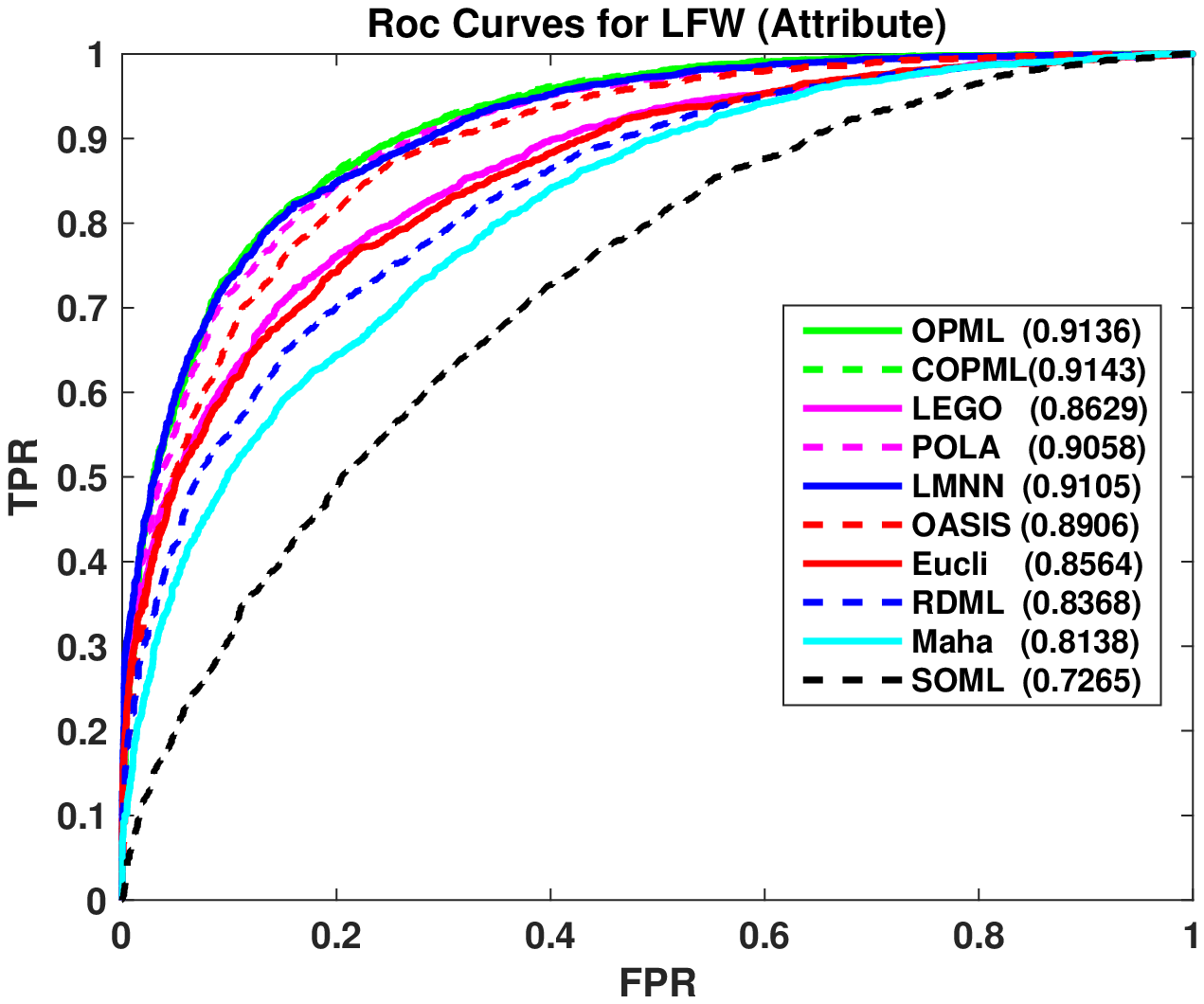}}
     \caption{ROC Curves of our methods and contrastive methods on the LFW dataset. left: sift features; right: attribute features. AUC value of each method is presented in bracket}
     \label{roc-Figure-lfw}
\end{figure}
\subsection{Face Verification: LFW}

For face verification, we also evaluate our methods on the Labeled Faces in the Wild Database (LFW) \cite{LFWTech}. LFW is a widely used face verification benchmark with unconstrained images, which contains 13233 images of 5749 individuals. This dataset has two views: View 1 is used for development purposes (containing a training set and a test set); And, View 2 is taken as evaluation benchmark for comparison (i.e., a 10-fold cross-validation set). There are two forms of configuration in both views, that is, image restricted configuration and image unrestricted configuration. In the first formulation, the training information is restricted to the provided image pairs and additional information such as actual name information can not be used. In other words, we can only use the pairwise images for training without any label information can be used at all. While, in the second formulation, the actual name information (i.e., label information) can be used and as many pairs or triplets can be formulated as one desires. No matter which configuration we choose, the test procedure is the same (i.e., using pairwise images for testing).  In order to simulate the real online environment and because our methods and some methods (eg., OASIS \cite{chechik2010large}, SOML \cite{gao2014soml}) are triplet-based methods, we adopt the image unrestricted configuration to construct the experiment. We use View 1 for parameter tuning and then evaluate the performance of all the algorithms on each fold (300 intra- and 300 extra-personal pairs) in View 2. Other settings are similar with the ones on the PubFig dataset.

In this experiment, we adopt two types of features (i.e., SIFT features and attribute features) to represent each face image, respectively. The SIFT features are provided by Guillaumin et al. \cite{guillaumin2009you} by extracting SIFT descriptors \cite{lowe2004distinctive} at 9 fixed facial landmarks detected on a face, over three scales. Then we perform PCA algorithm to reduce the original 3456-dimensional feature to a 100-dimensional feature. Like PubFig, the attribute features of LFW are 73-dimensional 'high-level' features describing the nameable attributes of a face image \cite{kumar2009attribute}. To evaluate our methods and the contrastive methods, we report the ROC curves and AUC values of the corresponding methods (see Fig. \ref{roc-Figure-lfw}). The results of ITML \cite{davis2007information} aren't displayed for its difficulty of convergence in the training data. We can see that the proposed COPML method can achieve the-state-of-the-art performance compared with the contrastive metric learning methods. Especially, when using SIFT features, our methods can significantly improve the AUC value over the Euclidean distance by $13\%$ ($5.8\%$ with attribute features), showing the validity of the proposed methods. It is worth noting that some metric learning methods cannot even improve over the Euclidean distance, which has happened on the PubFig dataset. The reason why LMNN cannot achieve the best performance may be over-fitting for lacking of regularization.

\subsection{Abnormal Event Detection in Videos}

The performance of the proposed methods is also evaluated on UMN dataset for abnormal event detection. UMN dataset contains 3 different scenes with 7739 frames in total: Scene1 (1453 frames), Scene2 (4144 frames) and Scene3 (2142 frames). In UMN dataset, people walking around is considered as normal, while people running away is regarded as abnormal. The resolution of the video is $320 \times 240$. We divide each frame into $5 \times 4$ non-overlapping $64 \times 60$ patches. For each patch, the MHOF (Multi-scale Histogram of Optical Flow) feature \cite{cong2011sparse} was extracted from every two successive frames. The MHOF is a 16-dimensional feature, which can capture both motion direction and motion energy. For integrating the multi-patches features, we combine features from all patches in each frame, and form a 320-dimensional feature. For each scene, we perform 2-fold cross validation for evaluation. The distance metric is learnt from the training data in online way, then we use the SVM classifier to classify the testing frames after feature transformation by using the learned metric $\bm{L}$.

Table \ref{UMN-AUC} reports the AUC of all the methods. We can notice that our methods is very effective and competitive, when compared with other methods. Fig. \ref{UMN-Figure} exhibits the sample frames of normal and abnormal events in the 3 scenes respectively (top row), and shows the abnormal event detection results of our method (COPML) in the indication bars (green/red indicates normal/abnormal event). It's worth mentioning that in this experiment, COPML performs better than OPML, because the video data has the cold start issue especially at the beginning.
%---------- UMN Results -----------
%\vspace{-10pt}
\begin{figure*}[!htbp]
     \centering
     \includegraphics[width = 1\textwidth]{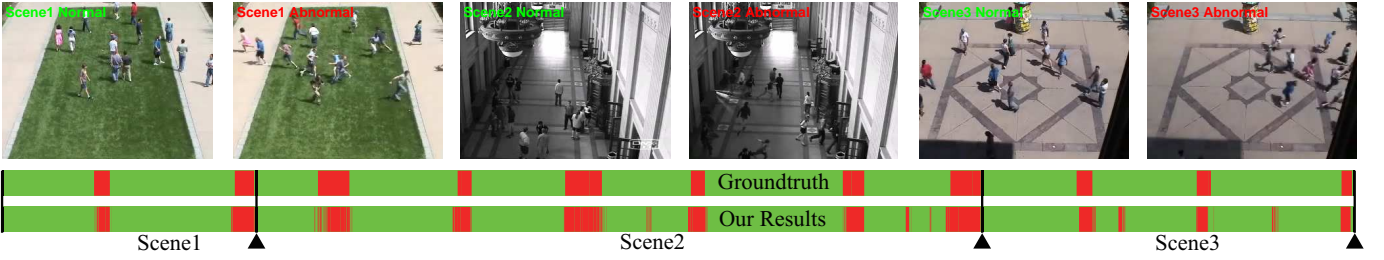}
     %\vspace{-15pt}
     \caption{Global abnormal event detection results of our method COPML and the ground truth on the UMN dataset}
     \label{UMN-Figure}
     %\vspace{-5pt}
\end{figure*}

%---------- AUC in UMN Data Set ---------
\begin{table}[!htbp]%\small
\begin{center}
\centering
%\tabcolsep=8pt
%\extrarowheight=0.03pt
%\vspace{-12pt}
\caption{Quantitative comparison of our methods with other abnormal event detection methods on 3 scenes of UMN dataset individually with AUC criterion}
\vspace{5pt}
\begin{tabular}{|l|c|}
\hline
Method & AUC\\
\hline\hline
Optical flow            \cite{mehran2009abnormal} & $0.84$ (average)\\
Social Force            \cite{mehran2009abnormal} & $0.96$ (average)\\
Chaotic Invariants      \cite{wu2010chaotic}      & $0.99$ (average)\\
LSA                     \cite{saligrama2012video} & $0.985$(average)\\
STCOG                   \cite{shi2010real}        & $0.936/0.776/0.966$\\
Sparse                  \cite{cong2011sparse}     & $0.995/0.975/0.964$\\
MP-MIDL                 \cite{huo2014multi}       & $0.99\ /0.98\ /0.99\ $\\
SVDD-based              \cite{zhang2016combining} & $0.993/0.969/0.988$\\
\hline
\textbf{OPML}                                     & $\textbf{0.993}/\textbf{0.983}/\textbf{0.973}$\\
\textbf{COPML}                                    & $\textbf{0.995}/\textbf{0.989}/\textbf{0.977}$\\
\hline
\end{tabular}
\label{UMN-AUC}
\end{center}
%\vspace{-16pt}
\end{table}

\subsection{COPML for Cold Start}
We can observe that in the case free of the cold start issue (e.g., UCI data classification, face verification), OPML and COPML can obtain comparable results, while in the case with cold start issue (e.g., abnormal event detection in videos), COPML is better than OPML. To further test the performance of COPML on an extreme case with cold start issue, we construct several datasets with specified structure to verify the different performance of COPML and OPML. Three datasets were picked from the UCI repository: (1) Image Segmentation (seg for short), with 7 classes, 19 features and 2310 samples; (2) EEG Eye State (eeg for short), with 2 classes, 15 features and 14980 samples; (3) Sensorless (sen for short), with 11 classes, 49 features and 58509 samples.

For each dataset, the samples from different classes are divided into disjoint 10/5/2 parts, then different parts of different classes are crosswise put together to construct a new dataset. Afterwards, the new dataset is divided into 2 folds. The first fold is used for training and the second fold is used for testing. As the previous setting for classification, we take a $k$-NN (k=5) classifier to get the final test results, shown in Table \ref{ver-errorrates}. The results prove that COPML performs better than OPML when the data has the cold start issue. Since when cold start occurs, COPML will incorporate both the pair and triplet information, instead of only using triplet in OPML.

\begin{table}[!htbp]%\small
\centering
%\renewcommand{\arraystretch}{0.97}
%\tabcolsep=20pt
%\extrarowheight=5pt
%\vspace{-8pt}
\caption{Error rates on three UCI datasets}
\vspace{5pt}
\begin{tabular}{|lccc|}
  \hline
  Data       &Euclidean            &OPML                      &COPML \\
  \hline\hline
  seg-10     &$0.069\!\pm\!0.006$  &$0.062\!\pm\!0.006$       &$\bm{0.057\!\pm\!0.007}$\\
  seg-5      &$0.067\!\pm\!0.007$  &$0.062\!\pm\!0.007$       &$\bm{0.054\!\pm\!0.007}$\\
  seg-2      &$0.067\!\pm\!0.006$  &$0.064\!\pm\!0.007$       &$\bm{0.059\!\pm\!0.007}$\\
  \hline
  eeg-10     &$0.185\!\pm\!0.004$  &$0.181\!\pm\!0.004$       &$\bm{0.161\!\pm\!0.023}$\\
  eeg-5      &$0.213\!\pm\!0.004$  &$0.205\!\pm\!0.005$       &$\bm{0.185\!\pm\!0.018}$\\
  eeg-2      &$0.185\!\pm\!0.004$  &$0.178\!\pm\!0.007$       &$\bm{0.178\!\pm\!0.007}$\\
  \hline
  sen-10     &$0.190\!\pm\!0.002$  &$0.093\!\pm\!0.036$       &$\bm{0.071\!\pm\!0.011}$\\
  sen-5      &$0.196\!\pm\!0.002$  &$0.097\!\pm\!0.032$       &$\bm{0.082\!\pm\!0.014}$\\
  sen-2      &$0.190\!\pm\!0.006$  &$0.067\!\pm\!0.020$       &$\bm{0.063\!\pm\!0.015}$\\
  \hline
\end{tabular}
\label{ver-errorrates}
%\vspace{-15pt}
\end{table}
\section{Conclusion}
\label{conclusions}
We propose a one-pass closed-form solution for OML, namely OPML. It employs the one-pass triplet construction for fast triplet generation, together with a closed-form solution to update the metric with the new coming sample at each time step. Also, for cold start issue, COPML, an extended version of OPML is developed. The major advantages of our methods are: OPML and COPML are easy to implement. Also, OPML and COPML are very scalable with low space (i.e., $O(d)$) and time (i.e., $O(d^2)$) complexity. In the experiments, we show that our methods can obtain superior performance on three typical tasks, compared with the state-of-the-art methods.

\appendices
\section{Proof of Theorem 2}
\begin{proof}
Recall that the metric update formula of OPML is
\begin{equation}\label{fun-p1}
   \bm{L}_t=\left\{
   \begin{array}{ll}
   \bm{L}_{t-1}(\bm{I}+\gamma\bm{A}_t)^{-1} & [z]_+>0\\
   \\
   \bm{L}_{t-1} & [z]_+=0.
   \end{array}
   \right.
\end{equation}
According to the Theorem 1, we can obtain that,
\begin{equation}\label{fun-p2}
     (\bm{I}+\gamma\bm{A}_t)^{-1}=\bm{I}-\frac{1}{\eta+\beta}[\eta\gamma\bm{A}_t-(\gamma\bm{A}_t)^2],
\end{equation}
where $\eta=1+tr(\gamma\bm{A}_t)$, $\beta=\frac{1}{2}[(tr(\gamma\bm{A}_t))^2-tr(\gamma\bm{A}_t)^2]$. Here, we only consider the case that $[z]_+>0$. Then at $t$-th time step, the learned metric $\bm{L}_t$ of one-pass strategy can be expressed as below,
\begin{equation}\label{fun-p3}
   \bm{L}_t=\bm{L}_0(\bm{I}+\gamma\bm{A}_1)^{-1}(\bm{I}+\gamma\bm{A}_{2})^{-1}\cdots(\bm{I}+\gamma\bm{A}_t)^{-1}.
\end{equation}
Note that the batch triplet construction strategy here is considered in an online manner, that is to say, for each sample $\bm{x}_t$ at the $t$-th time step, all past samples are stored to construct a triplet with $\bm{x}_t$ (i.e., each triplet contains this $\bm{x}_t$). Similar to $\bm{L}_t$, the learned metric $\bm{L}_t^\ast$ of the batch strategy (at $t$-th time step, $C_i|_{i=1}^t$ triplets can be constructed) can be denoted as follows,
\begin{equation}
   \bm{L}_t^\ast\!=\!\bm{L}_0^\ast\prod_{i=1}^{C_1}(\bm{I}\!+\!\gamma\bm{A}_{1_i})^{-1}\prod_{i=1}^{C_{2}}(\bm{I}\!+\!\gamma\bm{A}_{2_i})^{-1}\cdots\prod_{i=1}^{C_t}(\bm{I}\!+\!\gamma\bm{A}_{t_i})^{-1}.
\label{fun-p4}
\end{equation}
Let $\langle\bm{x}_1, \bm{x}_{p_1}, \bm{x}_{q_1}\rangle,\dots,\langle\bm{x}_t, \bm{x}_{p_t}, \bm{x}_{q_t}\rangle$ be the sequence of triplets constructed by the proposed one-pass strategy, which is contained in the sequence of triplets constructed by the batch strategy. If we let the $\bm{L}^\ast$ learn on the sequence of triplets constructed by the one-pass strategy first, the Eq. (\ref{fun-p4}) can be reorganized as below,
\begin{equation}\label{fun-p5}
 \begin{split}
    \bm{L}_t^\ast\!=\!&\bm{L}_0^\ast(\bm{I}\!+\!\gamma\bm{A}_1)^{-1}\!\cdots\!(\bm{I}\!+\!\gamma\bm{A}_t)^{-1}\!\cdot\!\prod_{i=1}^{C_1+\cdots+C_t-t}(\bm{I}\!+\!\gamma\bm{A}_i)^{-1}\\
                  &\small\text{($\bm{L}_t^\ast$ learn on the sequence of $\bm{L}_t$ first)}\\
                 =&\bm{L}_t\cdot\prod_{i=1}^{C_1+\cdots+C_t-t}(\bm{I}+\gamma\bm{A}_i)^{-1}\\
                  &\small\text{($\bm{L}_0$ and $\bm{L}_0^\ast$ are both initialized as identity matrices)}\\
                 =&\bm{L}_t\cdot\prod_{i=1}^{C_1+\cdots+C_t-t}(\bm{I}+\bm{B})\\
                  &\small\text{(by Theorem 1, where $\scriptstyle\bm{B}=\frac{1}{\eta+\beta}\Big[(\gamma\bm{A}_i)^2-\eta\gamma\bm{A}_i$\Big] )}\\
                 =&\bm{L}_t\Big[\bm{I}+\sum_{i=1}^{C_N}\bm{B}_i+\sum_{i=1,j=1,i<j}^{C_N}\bm{B}_i\bm{B}_j+\cdots+\prod_{i=1}^{C_N}\bm{B}_i \Big]\\
                  &\small\text{(where $C_N=C_1+\cdots+C_t-t$)}.\\
 \end{split}
\end{equation}
Then we can calculate that
\begin{small}
\begin{equation}
   \begin{split}
     \|\bm{L}_t\!-\!\bm{L}_t^\ast\|_F&\!=\!\Big\|\bm{L}_t\Big[\sum_{i=1}^{C_N}\bm{B}_i\!+\!\sum_{i=1,j=1,i<j}^{C_N}\bm{B}_i\bm{B}_j\!+\!\cdots\!+\!\prod_{i=1}^{C_N}\bm{B}_i\Big]\Big\|_F\\
                                 &\!\le\!\|\bm{L}_t\|_F\!\cdot\!\Big\|\sum_{i=1}^{C_N}\bm{B}_i\!+\!\sum_{i=1,j=1,i<j}^{C_N}\bm{B}_i\bm{B}_j\!+\!\!\cdots\!\!+\!\prod_{i=1}^{C_N}\bm{B}_i\Big\|_F.
   \end{split}
\label{fun-p6}
\end{equation}
\end{small}
Recall that $\bm{A}_t=\bm{M}_1-\bm{M}_2=(\bm{x}_t-\bm{x}_p)(\bm{x}_t-\bm{x}_p)^T-(\bm{x}_t-\bm{x}_q)(\bm{x}_t-\bm{x}_q)^T\in\mathbb{R}^{d\times d}$, which is a symmetry square matrix. According to the definition of Frobenius norm,
\begin{equation}
    \|\bm{A}_t\|_F=\sqrt{\sum_{i=1}^d\sum_{i=1}^d|a_{ij}|^2}=\sqrt{\sum_{i=1}^d\sigma_i^2},
\label{fun-p7}
\end{equation}
where $\sigma_i$ are the singular values of $\bm{A}_t$, which are equal to the eigenvalues of $\bm{A}_t$. According to Lemma 1, $-\lambda_\text{max}(\bm{M}_2)\le\lambda(\bm{A}_t)\le\lambda_\text{max}(\bm{M}_1)$, where $\lambda(\bm{A}_t)$ denotes the eigenvalue of $\bm{A}_t$, and $\lambda_\text{max}(\bm{M})$ indicates the maximum eigenvalue of $\bm{M}$. Assuming that $\|\bm{x}_t\|_2\le R$, then $\lambda_\text{max}(\bm{M}_1)$ belongs to the range of $[0, 4R^2]$. And since the rank of $\bm{A}_t$ is 2 (which has been proved in section \ref{OPML-section}), there are at most two nonzero eigenvalues. Thus we can easily obtain that $\|\bm{A}\|_F\le4\sqrt{2}R^2$. Hence,
\begin{equation}
   \begin{split}
    \|\bm{B}\|_F&=\|\frac{\gamma^2}{\eta+\beta}\bm{A}_t^2+\frac{\eta\gamma}{\eta+\beta}(-\bm{A}_t)\|_F\\
                &\le\|\frac{\gamma^2}{\eta+\beta}\bm{A}_t^2\|_F+\|\frac{\eta\gamma}{\eta+\beta}(-\bm{A}_t)\|_F\\
                &\le\Big|\frac{\gamma^2}{\eta+\beta}\Big|\cdot\|\bm{A}_t\|_F\cdot\|\bm{A}_t\|_F+\Big|\frac{\eta\gamma}{\eta+\beta}\Big|\cdot\|\bm{A}_t\|_F\\
                &\le32\Big|\frac{\gamma^2}{\eta+\beta}\Big|R^4+4\sqrt{2}\Big|\frac{\eta\gamma}{\eta+\beta}\Big|R^2.\\
   \end{split}
\label{fun-p8}
\end{equation}
Then, we can also calculate the range of $\eta$ and $\beta$ respectively.
\begin{equation}
   \begin{split}
    \eta &=1+tr(\gamma\bm{A}_t)\\
         &=1+\gamma\cdot tr(\bm{A}_t)\\
         &=1+\gamma\Big[(\bm{x}_t-\bm{x}_p)^T(\bm{x}_t-\bm{x}_p)-(\bm{x}_t-\bm{x}_q)^T(\bm{x}_t-\bm{x}_q)\Big]\\
         &=1+\gamma\Big[\|\bm{x}_p\|_2^2-\|\bm{x}_q\|_2^2-2\|\bm{x}_t^T\|_2\cdot\|\bm{x}_p\|_2\cos\theta_1\\
         &+2\|\bm{x}_t^T\|_2\cdot\|\bm{x}_q\|_2\cos\theta_2\Big].
   \end{split}
\label{fun-p9}
\end{equation}

Since the range of $\gamma$ is $(0,\frac{1}{4})$, and $0\!<\!\|\bm{x}_t\|_2\!\le\!R$, we can calculate to get the range of $tr(\gamma\bm{A}_t)$ (i.e., $(-\frac{5}{4}R^2,\frac{5}{4}R^2)$), and the range of $\eta$ which is $(1\!-\!\frac{5}{4}R^2,1\!+\!\frac{5}{4}R^2)$. Recall that,
\begin{equation}
   \begin{split}
       \beta=&\frac{1}{2}\left[(tr(\gamma\bm{A}_t))^2-tr(\gamma\bm{A}_t)^2\right]\\
             &\small\text{( by the rule of $\scriptstyle tr(c\bm{A})=c\cdot tr(\bm{A})$ )}\\
            =&\frac{1}{2}\left[(tr(\gamma\bm{A}_t))^2-\gamma^2tr(\bm{A}_t^2)\right]\\
            &\small\text{( by the rule of $\scriptstyle tr(\bm{A}^k)=\sum_i\lambda_i^k$,}\\
            &\small\text{where $\lambda_i$ is the eigenvalue of $\bm{A}$)}\\
            =&\frac{1}{2}\Big[(tr(\gamma\bm{A}_t))^2-\gamma^2\sum_{i=1}^d\lambda_i^2\Big]\\
            &\small\text{( by the rule of $\scriptstyle\|\bm{A}_t\|_F=\sqrt{\sum_{i=1}^d\lambda_i^2}$)}\\
            =&\frac{1}{2}\Big[(tr(\gamma\bm{A}_t))^2-\gamma^2\|\bm{A}_t\|_F^2\Big]\\
   \end{split}
\label{fun-p10}
\end{equation}
For $0\!<\!\|\bm{A}_t\|_F^2\!\le\!32R^4$ and $-\frac{5}{4}R^2\!<\!tr(\gamma\bm{A}_t)\!<\!\frac{5}{4}R^2$, the range of $\beta$ is $(-R^4, \frac{25}{32}R^4)$.
\end{proof}

\section{Proof of Theorem 3}
\begin{proof}
By applying the one-pass triplet construction strategy, at the $t$-th time step, we can obtain one triplet $\langle\bm{x}_t, \bm{x}_p, \bm{x}_q\rangle$. While in the batch construction (all past samples will be stored), we can get a triplet set $\{\langle\bm{x}_t, \bm{x}_{p_i}, \bm{x}_{q_i}\rangle\}|_{i=1}^C$.
%For more reasonable comparison, the contrastive strategy is chosen as the fully %online construction strategy rather than the offline construction strategy (the number %of triplets is enormous).
The average loss of these two strategies can be expressed as follows:
\begin{equation}
  \begin{aligned}
   &\Psi_1=\Big\lbrack1+\|\bm{L}(\bm{x}_t-\bm{x}_p)\|_2^2-\|\bm{L}(\bm{x}_t-\bm{x}_q)\|_2^2\Big\rbrack_+\\
   &\Psi_2=\frac{1}{C}\sum_{i=1}^{C}\Big[1\!+\!\|\bm{L}^\ast(\bm{x}_t\!-\!\bm{x}_{p_i})\|_2^2\!-\!\|\bm{L}^\ast(\bm{x}_t\!-\!\bm{x}_{q_i})\|_2^2\Big]_+,
  \end{aligned}
\end{equation}
where $[z]_+=\max(0,z)$,  namely the hinge loss $\mathcal{G}((\bm{x}_t, \bm{x}_p, \bm{x}_q);\bm{L})$. For $\Psi_1$, we only consider the case that $z\ge0$, which exactly affects the updating of the metric $\bm{L}$. However, in $\Psi_2$, some losses may be negative. Thus,
\begin{equation}
  \begin{aligned}
   &\Psi_1=1+\|\bm{L}(\bm{x}_t-\bm{x}_p)\|_2^2-\|\bm{L}(\bm{x}_t-\bm{x}_q)\|_2^2\\
   &\Psi_2\le\frac{1}{C}\sum_{i=1}^{C}\Big[1+\|\bm{L}^\ast(\bm{x}_t-\bm{x}_{p_i})\|_2^2-\|\bm{L}^\ast(\bm{x}_t-\bm{x}_{q_i})\|_2^2\Big].
  \end{aligned}
\end{equation}
Then we calculate the difference between the losses of these two strategies. That is,
\begin{equation}
  \begin{split}
  \Psi_1-\Psi_2&\le\frac{1}{C}\sum_{i=1}^C\Delta\\
                       &\le\frac{1}{C}\sum_{i=1}^C\Big[\|\bm{L}(\bm{x}_t-\bm{x}_p)\|_2^2-\|\bm{L}^\ast(\bm{x}_t-\bm{x}_{p_i})\|_2^2\\
                       &+\|\bm{L}^\ast(\bm{x}_t-\bm{x}_{q_i})\|_2^2-\|\bm{L}(\bm{x}_t-\bm{x}_q)\|_2^2\Big].
  \end{split}
\end{equation}
For simplicity, we just analysis the $\Delta$. By applying the rule of dot product (i.e., $\bm{x}^T\bm{y}=\|\bm{x}\|\cdot\|\bm{y}\|\cdot\cos\theta$), $\Delta$ can be expanded as follows,
\begin{equation}
  \begin{split}
   \Delta&=\|\bm{L}\bm{x}_p\|_2^2-\|\bm{L}^\ast\bm{x}_{p_i}\|_2^2+\|\bm{L}^\ast\bm{x}_{q_i}\|_2^2-\|\bm{L}\bm{x}_q\|_2^2\\
                 &-2\|\bm{L}\bm{x}_t\|\!\cdot\!\|\bm{L}\bm{x}_p\|\!\cdot\!\cos\theta_1\!+\!2\|\bm{L}^\ast\bm{x}_t\|\!\cdot\!\|\bm{L}^\ast\bm{x}_{p_i}\|\!\cdot\!\cos\theta_2\\
                 &-2\|\bm{L}^\ast\bm{x}_t\|\cdot\|\bm{L}^\ast\bm{x}_{q_i}\|\cdot\cos\theta_3+2\|\bm{L}\bm{x}_t\|\cdot\|\bm{L}\bm{x}_q\|\cdot\cos\theta_4.
  \end{split}
\end{equation}
In order to simplify the expression, we set
\begin{equation}
  \begin{aligned}
   &\textcircled{1}=-2\|\bm{L}\bm{x}_t\|\cdot\|\bm{L}\bm{x}_p\|\cdot\cos\theta_1\\
   &\textcircled{2}=2\|\bm{L}^\ast\bm{x}_t\|\cdot\|\bm{L}^\ast\bm{x}_{p_i}\|\cdot\cos\theta_2\\
   &\textcircled{3}=-2\|\bm{L}^\ast\bm{x}_t\|\cdot\|\bm{L}^\ast\bm{x}_{q_i}\|\cdot\cos\theta_3\\
   &\textcircled{4}=2\|\bm{L}\bm{x}_t\|\cdot\|\bm{L}\bm{x}_q\|\cdot\cos\theta_4.
  \end{aligned}
\end{equation}
Assuming that the angle $\theta$ between two samples coming from the same class is very small after the transformation of $\bm{L}$ or $\bm{L}^\ast$ (i.e., $\cos\theta=\alpha, \alpha\ge0$ and $\alpha$ is close to 1), while $\theta$ is very large otherwise (i.e., $\cos\theta=-\xi, \xi\ge0$ and $\xi$ is close to 1). Thus,
\begin{equation}
  \begin{aligned}
   0&\le\cos\theta_1\le\alpha\\
   0&\le\cos\theta_2\le\alpha\\
   -\xi&\le\cos\theta_3\le0\\
   -\xi&\le\cos\theta_4\le0,
  \end{aligned}
\end{equation}
where $0\le\alpha\le1$, $0\le\xi\le1$ and both of them are close to 1. Then, we can obtain that,
\begin{equation}
  \begin{aligned}
   -2\alpha\|\bm{L}\bm{x}_t\|\cdot\|\bm{L}\bm{x}_p\|&\le\textcircled{1}\le0\\
   0&\le\textcircled{2}\le2\alpha\|\bm{L}^\ast\bm{x}_t\|\cdot\|\bm{L}^\ast\bm{x}_{p_i}\|\\
   0&\le\textcircled{3}\le2\xi\|\bm{L}^\ast\bm{x}_t\|\cdot\|\bm{L}^\ast\bm{x}_{q_i}\|\\
   -2\xi\|\bm{L}\bm{x}_t\|\cdot\|\bm{L}\bm{x}_q\|&\le\textcircled{4}\le0.
  \end{aligned}
\end{equation}
Here, we only consider the upper bound of $\Delta$,
\begin{small}
\begin{equation}
  \begin{split}
   \Delta&\le\|\bm{L}\bm{x}_p\|_2^2-\|\bm{L}^\ast\bm{x}_{p_i}\|_2^2+\|\bm{L}^\ast\bm{x}_{q_i}\|_2^2-\|\bm{L}\bm{x}_q\|_2^2\\
   &\quad+2\alpha\|\bm{L}^\ast\bm{x}_t\|\cdot\|\bm{L}^\ast\bm{x}_{p_i}\|+2\xi\|\bm{L}^\ast\bm{x}_t\|\cdot\|\bm{L}^\ast\bm{x}_{q_i}\|\\
                 &\le\|\bm{L}\bm{x}_p\|_2^2+\|\bm{L}^\ast\bm{x}_{q_i}\|_2^2+2\alpha\|\bm{L}^\ast\bm{x}_t\|\cdot\|\bm{L}^\ast\bm{x}_{p_i}\|\\
                 &\quad+2\xi\|\bm{L}^\ast\bm{x}_t\|\cdot\|\bm{L}^\ast\bm{x}_{q_i}\|.
  \end{split}
\end{equation}
\end{small}
According to the property of compatible norms, that is,
\begin{equation}
   \|\bm{A}\bm{x}\|_2\le\|\bm{A}\|_F\cdot\|\bm{x}\|_2.
\end{equation}
For assuming that $\|\bm{x}\|_2\le R$ (for all samples), $\|\bm{L}\|_F\le U$ and $\|\bm{L}^\ast\|_F\le U$, we can obtain that,
\begin{equation}
  \begin{aligned}
    \Delta&\le2(\alpha+\xi+1)R^2U^2\\
    \Psi_1-\Psi_2&\le2(\alpha+\xi+1)R^2U^2.
  \end{aligned}
\end{equation}
Thus, this theorem has be proved.
\end{proof}

\section{Proof of Theorem 4}
\begin{proof}
The regret can be defined (according to the definition of Chapter 3 in \cite{shalev2007online}) as below:
\begin{equation}\label{fun-p11}
     R(\bm{L}_\ast,T)=\sum_{t=1}^T\mathcal{G}_t(\bm{L}_t)-\sum_{t=1}^T\mathcal{G}_t(\bm{L}_\ast),
\end{equation}
where $\mathcal{G}_t(\bm{L})=\lbrack1+\Vert\bm{L}(\bm{x}_t-\bm{x}_p)\Vert_2^2-\Vert\bm{L}(\bm{x}_t-\bm{x}_q)\Vert_2^2  \rbrack_+$. Here, we also only consider the case that the loss is positive, which exactly affects the updating of the metric $\bm{L}$. However, one triplet which generates a positive loss with $\bm{L}$ may incur a negative loss with $\bm{L}_\ast$. Thus, after expanding,
\begin{equation}\label{fun-p12}
\begin{aligned}
     R(\bm{L}_\ast,T)&\le\sum_{t=1}^T\Big[\Vert\bm{L}_t(\bm{x}_t-\bm{x}_p)\Vert_2^2-\Vert\bm{L}_t(\bm{x}_t-\bm{x}_q)\Vert_2^2\\
     &-\Vert\bm{L}_\ast(\bm{x}_t-\bm{x}_p)\Vert_2^2+\Vert\bm{L}_\ast(\bm{x}_t-\bm{x}_q)\Vert_2^2\Big].
\end{aligned}
\end{equation}
In the similar way of proving the Theorem \ref{theorem3}, we can easily prove this theorem.
\end{proof}

%\section*{Acknowledgment}

\ifCLASSOPTIONcaptionsoff
  \newpage
\fi



%\bibliographystyle{IEEEtran}
%\begin{thebibliography}{1}


%\end{thebibliography}

\bibliographystyle{IEEEtran}
\bibliography{IEEEabrv,IEEEexample}



%\vfill


\end{document}